\documentclass[11pt,a4paper]{article}
\usepackage[margin=25mm]{geometry}
\usepackage{lmodern}
\usepackage[sc,noBBpl]{mathpazo}
\usepackage[authoryear,round]{natbib}
\setcitestyle{citesep={;},aysep={,},yysep={;}}
\usepackage[T1]{fontenc}
\usepackage[utf8]{inputenc}
\usepackage{microtype}
\usepackage[table]{xcolor}
\usepackage{graphicx}
\usepackage{wrapfig}
\usepackage{needspace}
\usepackage{booktabs,multirow,makecell,array,tabularx,longtable}
\usepackage{amsmath,amssymb,mathtools,amsthm}
\usepackage{pifont}
\usepackage{caption,subcaption}
\usepackage{algorithm,algorithmic}
\usepackage{listings}
\usepackage[most]{tcolorbox}
\usepackage{placeins}
\usepackage{hyperref}
\usepackage{xurl}
\usepackage[capitalize,noabbrev]{cleveref}

\definecolor{citeblue}{rgb}{0,0.08,0.45}
\definecolor{linkred}{HTML}{990000}
\hypersetup{colorlinks=true,allcolors=black,citecolor=citeblue,linkcolor=linkred,pdfauthor={Liangyu Teng, Hengsong Liu, Juncen Guo, Jingyu Zhang, Yang Liu, Jing Liu, Liang Song},pdftitle={Which Models Work Well Together? Measuring Heterogeneity for LLM Team Selection}}
\AddToHook{env/wraptable/before}{\Needspace{12\baselineskip}}

\makeatletter
\renewcommand{\@seccntformat}[1]{\csname the#1\endcsname.\hspace{0.5em}}
\renewcommand{\section}{\@startsection{section}{1}{\z@}%
  {-2.5ex plus -0.5ex minus -0.2ex}{1ex plus 0.2ex}%
  {\normalfont\fontsize{13}{16}\selectfont\bfseries}}
\renewcommand{\subsection}{\@startsection{subsection}{2}{\z@}%
  {-2ex plus -0.5ex minus -0.2ex}{0.8ex plus 0.2ex}%
  {\normalfont\fontsize{12}{15}\selectfont\bfseries}}
\renewcommand{\subsubsection}{\@startsection{subsubsection}{3}{\z@}%
  {-1.5ex plus -0.5ex minus -0.2ex}{0.6ex plus 0.2ex}%
  {\normalfont\normalsize\bfseries}}
\renewcommand{\paragraph}{\@startsection{paragraph}{4}{\z@}%
  {1.25ex plus 0.4ex minus 0.2ex}{-1em}%
  {\normalfont\normalsize\bfseries}}
\makeatother

\definecolor{abstractcream}{HTML}{F7F0E3}
\renewenvironment{abstract}{%
  \begin{tcolorbox}[enhanced,breakable,colback=abstractcream,colframe=abstractcream,
    boxrule=0pt,arc=5pt,boxsep=0pt,left=14pt,right=14pt,top=12pt,bottom=12pt,
    before skip=6pt,after skip=12pt]
  \normalfont\normalsize\setlength{\parindent}{0pt}\ignorespaces
}{\end{tcolorbox}}

\makeatletter
\renewcommand{\@maketitle}{%
  \begingroup
  \setlength{\parindent}{0pt}\setlength{\parskip}{0pt}
  \vspace*{9pt}\nointerlineskip
  {\centering\fontsize{16}{19}\selectfont\bfseries\@title\par}
  \vspace{12pt}
  {\centering\normalsize\bfseries Fudan Institute on Networking Systems of AI\par}
  \vspace{18pt}
  \endgroup
}
\makeatother

\newcommand{\finishwrap}{%
  \par
  \ifnum\value{WF@wrappedlines}>1
    \vspace{\dimexpr\baselineskip*(\value{WF@wrappedlines}-1)\relax}%
  \fi
  \WFclear
}

\definecolor{scoregreen}{RGB}{0,120,0}
\definecolor{scorered}{RGB}{170,0,0}
\definecolor{scoregray}{RGB}{89,89,89}
\newcommand{\better}[1]{\textcolor{scoregreen}{#1}}

\newcommand{\scoregain}[2]{#1\raisebox{-0.35ex}{\fontsize{6}{6}\selectfont\better{$\uparrow$#2}}}
\newcommand{\modelicon}[1]{\raisebox{-0.18em}{\includegraphics[height=1.15em]{#1}}}
\newcommand{\modeltag}[2]{\modelicon{#1}~#2}
\newcommand{\removedtag}[2]{\colorbox{red!7}{\modeltag{#1}{#2}}}
\newcommand{\addedtag}[2]{\colorbox{green!9}{\modeltag{#1}{#2}}}
\newcommand{\xmark}{\ding{55}}
\newcommand{\ind}{\mathbf{1}}

\theoremstyle{plain}
\newtheorem{theorem}{Theorem}[section]
\newtheorem{proposition}[theorem]{Proposition}
\newtheorem{lemma}[theorem]{Lemma}
\newtheorem{corollary}[theorem]{Corollary}
\theoremstyle{definition}

\theoremstyle{remark}
\newtheorem{remark}{Remark}

\newcommand{\cmas}{C\textsuperscript{2}-MAS}
\newcommand{\HIerr}{\mathrm{HI}_{\mathrm{err}}}
\newcommand{\HIdist}{\mathrm{HI}_{\mathrm{dist}}}

\newcommand{\JSD}{\mathrm{JSD}}

\DeclareMathOperator*{\argmax}{arg\,max}

\newcommand{\E}{\mathbb{E}}
\renewcommand{\Pr}{\mathbb{P}}

\title{Which Models Work Well Together? Measuring Heterogeneity for LLM Team Selection}
\newcommand{\paperauthors}{%
  Liangyu Teng\textsuperscript{1},
  Hengsong Liu\textsuperscript{1},
  Juncen Guo\textsuperscript{1},
  Jingyu Zhang\textsuperscript{1},
  Yang Liu\textsuperscript{2},
  Jing Liu\textsuperscript{3},
  Liang Song\textsuperscript{1}%
}
\author{\paperauthors}
\date{}
\begin{document}
\maketitle
\begin{abstract}
The performance ceiling of an LLM team is constrained not only by individual model capabilities, but also by inter-member error resonance and predictive differences. Although heterogeneous teaming is often observed to be effective in practice, existing approaches lack complementarity metrics that are computable, interpretable, and optimizable, leaving team composition to rely on heuristics. We propose a heterogeneity-driven team selection framework that performs offline profiling to characterize individual capability along with two complementary signals: one captures decorrelation in error patterns to reduce co-failures, while the other measures divergence in predictive behavior to capture strategy diversity. We formulate team selection as a standardized quality--complementarity combinatorial objective and apply an efficient greedy search to select a small team from a candidate pool. Experiments across multiple benchmarks demonstrate that our framework consistently outperforms quality-only baselines under controlled candidate pools and team sizes, establishing reusable selection principles for multi-LLM systems.
\end{abstract}

\section{Introduction}
\label{sec:introduction}

Large language models (LLMs) are increasingly deployed as \emph{teams}: multiple models or agents collaborate through voting, debate, or orchestration to improve reasoning accuracy and robustness \citep{du2024improving, liang2024encouraging, wang2025mixture}.
In practical deployments, however, collaboration gains are tightly constrained by inference budget: each additional member introduces extra model calls and tokens, increasing end-to-end latency and cost.
Real systems must therefore keep teams small, drawing from a limited pool of deployable candidates.
These constraints make \emph{member selection} a primary design decision: under a fixed budget, which models should we call so that the team is both individually strong and behaviorally complementary?

A natural starting point is to pick the top-$k$ models by individual accuracy, but as illustrated in Figure~\ref{fig:motivation}(a), this strategy ignores a second axis that bounds team performance: \emph{similarity across members}.
Models can exhibit error resonance, failing on the same instances and concentrating probability on the same incorrect options; even when they agree on a label, they may share highly similar predictive behavior (confidence and preferences over alternatives), offering little strategy diversity.
Top-$k$-by-accuracy selection therefore tends to yield a strong yet redundant team whose blind spots largely overlap.
These two sources of redundancy (\textbf{co-failure} and \textbf{predictive similarity}) motivate selection criteria that go beyond quality and capture inter-model complementarity directly.

\begin{figure}[t]
    \centering
    \includegraphics[width=0.95\textwidth]{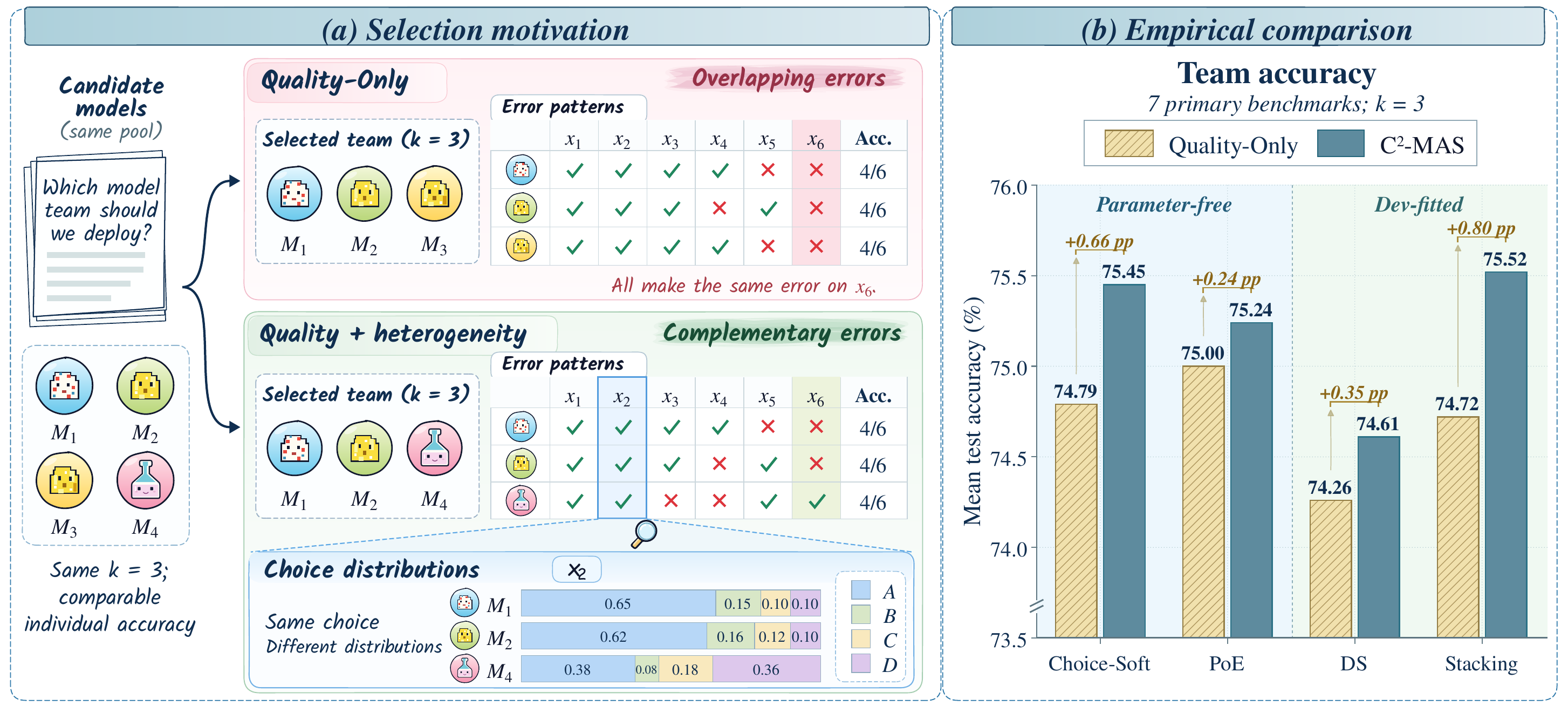}
    \caption{Illustration of heterogeneity-aware team selection.}
    \label{fig:motivation}
\end{figure}

While the value of heterogeneous teaming is widely recognized, practical team composition still relies largely on heuristics such as mixing model families.
Most recent work has instead focused on \emph{downstream} combination mechanisms such as output fusion, reranking, routing, and cascading, all of which take a callable model set and budget as given \citep{jiang2023llmblender, chen2024frugalgpt, shnitzer2024large, ong2025routellm}.
The \emph{upstream} question of which models to select in the first place remains comparatively underexplored.
We address this gap by formulating team selection as an offline, budgeted optimization problem: composing a compact and complementary team from development data, independently of the inference-time aggregation rule.
This requires complementarity signals that are computable, interpretable, and tractable to optimize, together with a selection procedure that scales to combinatorial search over candidate teams.

Driven by these requirements, we propose \textbf{C$^{2}$-MAS} (\textbf{C}ovariance-minimized and \textbf{C}apability-maximized \textbf{M}ulti-\textbf{A}gent \textbf{S}ystem), an offline profiling-and-selection pipeline that directly targets the two sources of redundancy identified above.
Given development data, we profile each candidate to estimate its individual capability together with two complementary heterogeneity signals:
\textbf{(1) error decorrelation}, which addresses co-failure by capturing whether models fail on different instances; and
\textbf{(2) distributional divergence}, which addresses predictive similarity by measuring disagreement in choice distributions over alternatives.
When multiple tasks are available, we average pairwise estimates across tasks to stabilize the complementarity signal.
We then select a size-$k$ team by optimizing a standardized quality--complementarity objective with an efficient multi-start greedy search.

Experiments on 13 multiple-choice benchmarks with 11 open-weight models show that heterogeneity signals yield consistent gains across aggregation rules (Figure~\ref{fig:motivation}(b)).
On the 7 primary benchmarks, \cmas{} reaches 75.52\% accuracy under \textsc{Stacking}, improving over the strongest selection baseline (Quality-Only, 74.72\%) by +0.80 pp \emph{without regressing on any task}, and over the random-k baseline by +3.58 pp.
On 6 additional held-out benchmarks, the gains persist under distribution shift, suggesting that the proposed signals capture transferable complementarity rather than task-specific artifacts. Our key contributions are summarized as follows:

\ding{182}~\textbf{Problem Formulation:} We frame team selection as an upstream, budget-constrained optimization problem, decoupled from downstream aggregation mechanisms.

\ding{183}~\textbf{Practical Solution:} We propose two interpretable heterogeneity metrics and a standardized quality--complementarity objective with efficient greedy search. Our analysis connects the two signals to collective errors and identifies conditions for accuracy gains over quality-only selection.

\ding{184}~\textbf{Empirical Validation:} We demonstrate consistent, non-regressive improvements across multiple benchmarks and aggregation rules, with gains that persist on held-out tasks under distribution shift.

\section{Related Work}
\label{sec:related_work}

\subsection{Multi-Agent and Multi-Model Collaboration}
LLM-based multi-agent systems organize collaboration through predefined roles and workflows \citep{li2023camel, hong2024metagpt, wu2024autogen}, and refine interaction through debate and voting \citep{du2024improving, choi2025debate}, communication structure optimization, and dynamic team formation \citep{zhuge2024gptswarm, zhang2025gdesigner, liu2024a}. Beyond agent interaction, multi-model methods perform fusion at the output, reasoning, or decoding stage \citep{jiang2023llmblender, wang2025mixture, huang2024ensemble}, while routing and cascading select models per input to balance quality and cost \citep{chen2024frugalgpt, shnitzer2024large, ong2025routellm}. Distinct from these approaches, we focus on task-level offline team selection, constructing fixed-size teams whose membership is determined independently of the subsequent aggregation rule.

\subsection{Ensemble Diversity and Team Selection}
Classical ensemble theory emphasizes individual accuracy and low error correlation, motivating diversity measures and validation-driven ensemble selection \citep{hansen1990neural, krogh1994neural, kuncheva2003measures, windeatt2005diversity, caruana2004ensemble}. Recent LLM ensemble selection methods exploit error or semantic diversity \citep{tekin2024llmtopla, cohen2026dfpe}, or adopt information-theoretic and label-based surrogate objectives \citep{turkmen2026dont, zhang2026complementary}. Building on this line of work, we jointly quantify error correlation and distributional divergence through a single offline profiling pass, combine these signals with model quality for team selection, and evaluate their utility and generality across multiple aggregation rules.

\section{Methodology}
\label{sec:method}

\cmas{} selects compact LLM teams through offline profiling and team search, followed by inference-time aggregation (Figure~\ref{fig:cmas_pipeline}). For each task, development data provide quality and heterogeneity estimates (\Cref{sec:profiling_signals}) for the standardized selection objective (\Cref{sec:selection}); the selected team's predictions are then combined by an aggregator (\Cref{sec:aggregation}).

\subsection{Preliminaries}
\label{sec:preliminaries}

We consider a collection of tasks $\mathcal{T}$. Each task $t\in\mathcal{T}$ has a label set
$\mathcal{Y}_t$ (e.g., $\{A,B,C,D\}$) and two disjoint splits: a development set
$D_t^{\mathrm{dev}}$ and a held-out test set $D_t^{\mathrm{test}}$. We assume a fixed pool of $m$
candidate models $\mathcal{M}=\{1,\ldots,m\}$ and select a team $S_t\subseteq\mathcal{M}$ of size
$k$ for each task $t$. Throughout, $D_t^{\mathrm{dev}}$ is used only for candidate profiling,
complementarity estimation, and training an aggregator (if needed), while $D_t^{\mathrm{test}}$ is used only for final evaluation.

To obtain comparable choice distributions, we use a short prompt to elicit a single option from the
model: ``Complete the sentence `The correct answer is ' by outputting exactly one letter from:
\texttt{<options>}''.\footnote{See \Cref{app:scoring_implementation} for full details.} Using the fixed probe prefix ``The correct answer is '', for each instance $x$
we query the model's log-probability $\ell_i(y\mid x)$ for each label $y\in\mathcal{Y}_t$ (i.e., the
next-token log-probability under this prefix) and normalize within $\mathcal{Y}_t$:
\begin{equation}
    p_i^t(y\mid x)
    = \frac{\exp(\ell_i(y\mid x))}{\sum_{y'\in\mathcal{Y}_t}\exp(\ell_i(y'\mid x))}.
    \label{eq:scoring_distribution}
\end{equation}
By construction, $p_i^t(\cdot\mid x)$ is a proper distribution over $\mathcal{Y}_t$. Based on this
distribution, model $i$ predicts $\hat{y}_i(x)=\argmax_{y\in\mathcal{Y}_t} p_i^t(y\mid x)$, and we
define the correctness indicator $C_i(x)=\mathbb{I}[\hat{y}_i(x)=y(x)]$, which equals
$1$ if the prediction is correct and $0$ otherwise. These definitions provide a
unified interface for the subsequent quality estimation, complementarity signals, and team
selection.

\begin{figure}[!t]
    \centering
    \includegraphics[width=\textwidth]{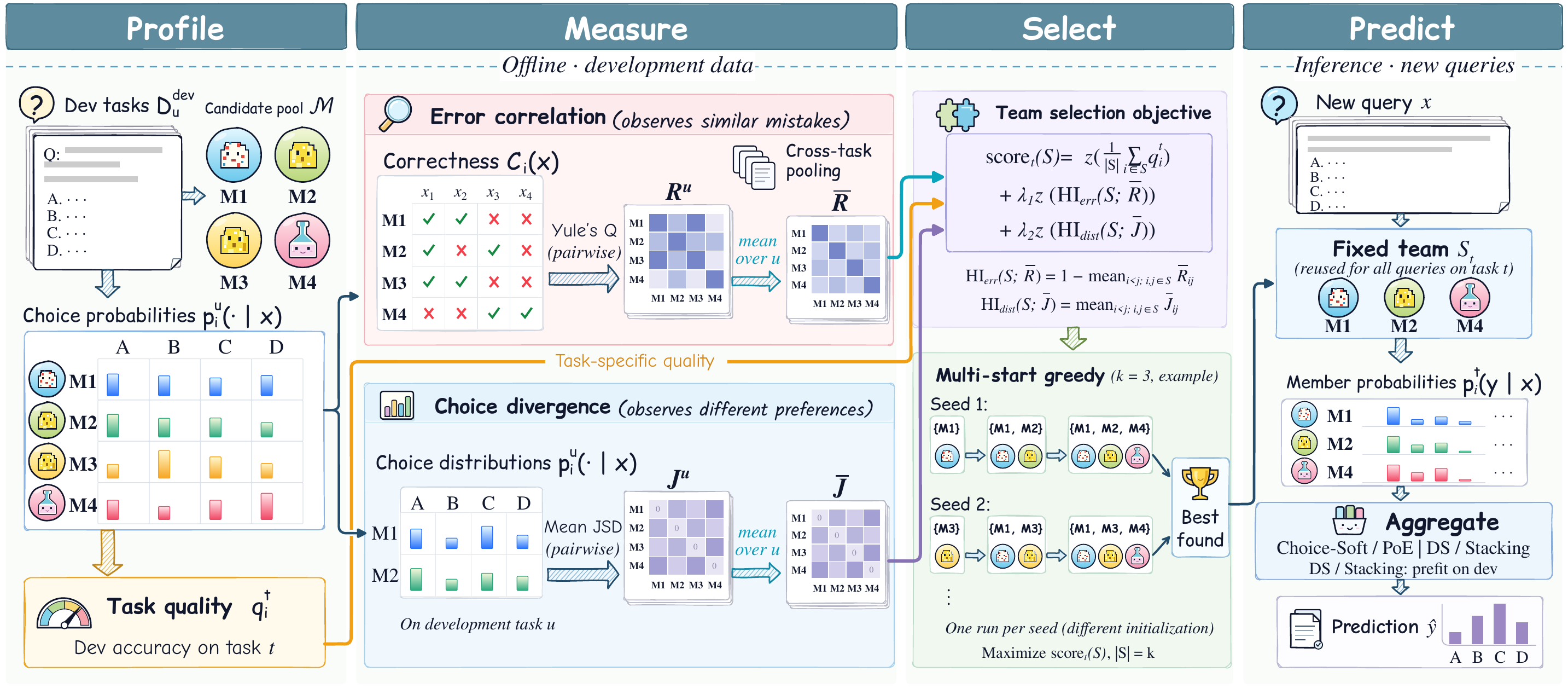}
    \caption{Overview of \cmas{} framework. Given a candidate pool of $m$ models, we (1)~profile each model on development data to obtain quality estimates and choice distributions; (2)~compute two heterogeneity index metrics ($\mathrm{HI}_{err}$ and $\mathrm{HI}_{dist}$) and stabilize them via cross-task pooling; (3)~select a team of size $k$ by optimizing a composite objective that balances quality and complementarity using multi-start greedy search; and (4)~aggregate the selected team's predictions to produce final outputs, evaluated on held-out test data.}
    \label{fig:cmas_pipeline}
\end{figure}

\subsection{Profiling Signals}
\label{sec:profiling_signals}

From $D_t^{\mathrm{dev}}$, we estimate per-model quality $\{q_i^t\}$ and two heterogeneity
index (HI) metrics: error decorrelation $\mathrm{HI}_{err}$ from an association matrix $R^t$,
and distributional divergence $\mathrm{HI}_{dist}$ from a choice-divergence matrix $J^t$.
These signals serve as inputs to the subsequent team selection procedure.

\paragraph{Quality.}
We measure the individual capability of model $i$ on task $t$ by its development-set accuracy,
denoted as $q_i^t$, computed as:
\begin{equation}
    q_i^t = \frac{1}{|D_t^{\mathrm{dev}}|}\sum_{x\in D_t^{\mathrm{dev}}} C_i(x).
    \label{eq:quality}
\end{equation}

\paragraph{Error Decorrelation ($\mathrm{HI}_{err}$).}
To capture complementarity at the error level, we summarize whether two models tend to succeed and
fail on the same instances. For a pair $(i,j)$, we form a $2\times 2$ contingency table on
$D_t^{\mathrm{dev}}$:
\begin{equation}
    N_{ab}^t(i,j)=\sum_{x\in D_t^{\mathrm{dev}}}\mathbb{I}\!\left[C_i(x)=a,\;C_j(x)=b\right],
    \label{eq:contingency}
\end{equation}
where $a,b\in\{0,1\}$. 
Let $D_{ij}^t = N_{11}^t N_{00}^t + N_{10}^t N_{01}^t$, we then compute Yule's $Q$ association statistic and define the pairwise error-correlation matrix
$R^t\in\mathbb{R}^{m\times m}$ as:
\begin{equation}
    R_{ij}^t =
    \begin{cases}
        \frac{N_{11}^t N_{00}^t - N_{10}^t N_{01}^t}{D_{ij}^t}, & D_{ij}^t > 0,\\
        0, & D_{ij}^t = 0.
    \end{cases}
    \label{eq:yule_q}
\end{equation}
where we write $N_{ab}^t \equiv N_{ab}^t(i,j)$ for readability.
Given a team $S$, we aggregate pairwise correlations into an error-decorrelation score:
\begin{equation}
    \mathrm{HI}_{err}(S;R^t)
    = 1 - \frac{1}{\binom{|S|}{2}}
    \sum_{\{i,j\}\subseteq S} R_{ij}^t .
    \label{eq:hi_err}
\end{equation}
Higher $\mathrm{HI}_{err}$ indicates that team members make mistakes on different instances.

\paragraph{Distributional Divergence ($\mathrm{HI}_{dist}$).}
Correctness correlations omit how models distribute probability over answer choices.
We therefore measure pairwise disagreement
directly in the choice distribution space using the Jensen--Shannon divergence (JSD) \citep{lin1991divergence,endres2003new}, computed with base-2 logarithms. For a pair
$(i,j)$, we average per-instance JSD on $D_t^{\mathrm{dev}}$ to obtain a matrix
$J^t\in\mathbb{R}^{m\times m}$:
\begin{equation}
    J_{ij}^t = \frac{1}{|D_t^{\mathrm{dev}}|}\sum_{x\in D_t^{\mathrm{dev}}}
    \mathrm{JSD}\!\left(p_i^t(\cdot\mid x),\, p_j^t(\cdot\mid x)\right).
    \label{eq:jsd_matrix}
\end{equation}
Given a team $S$, we aggregate pairwise divergences into a distributional-divergence score:
\begin{equation}
    \mathrm{HI}_{dist}(S;J^t) = \frac{1}{\binom{|S|}{2}}
    \sum_{\{i,j\}\subseteq S} J_{ij}^t .
    \label{eq:hi_dist}
\end{equation}
A larger $\mathrm{HI}_{dist}$ indicates greater disagreement in the models' choice distributions.

\paragraph{Cross-Task Pooling for Stable HI.}
Pairwise heterogeneity estimates can be noisy when $|D_t^{\mathrm{dev}}|$ is limited, which may
make team selection overly sensitive to idiosyncrasies of a single task's development split. When a
collection of tasks $\mathcal{T}$ is available, we pool pairwise heterogeneity across tasks to obtain a shared selection signal.
For each pairwise matrix, we compute the pooled estimates by averaging over tasks:
\begin{equation}
    \bar{R} = \frac{1}{|\mathcal{T}|}\sum_{u\in\mathcal{T}} R^{u},
    \qquad
    \bar{J} = \frac{1}{|\mathcal{T}|}\sum_{u\in\mathcal{T}} J^{u}.
    \label{eq:all_pooling}
\end{equation}
In our experiments, we use $\bar{R}$ and $\bar{J}$ as the heterogeneity signals for team selection.

\subsection{Heterogeneity-Aware Team Selection}
\label{sec:selection}

Given the development-set profiling results, our goal is to select a team
$S_t\subseteq\mathcal{M}$ of size $k$ that balances individual capability and inter-model
complementarity. We combine task-specific quality estimates $\{q_i^t\}$ with the cross-task
pooled heterogeneity matrices $\bar{R}$ and $\bar{J}$, which summarize correctness dependence
and differences in choice distributions, respectively.

We formulate team selection as maximizing a composite score comprising average member quality,
error decorrelation $\mathrm{HI}_{err}(S;\bar{R})$, and distributional divergence
$\mathrm{HI}_{dist}(S;\bar{J})$. Since these terms may have different scales across tasks and
candidate pools, we standardize each term with a standard-error-scaled $z$-score before
combining them.

Let $u(S)$ denote a team-level quantity computed as an average (e.g., the average quality,
$\mathrm{HI}_{err}(S;\bar{R})$, or $\mathrm{HI}_{dist}(S;\bar{J})$), and let $n_u(S)$ be the number of
items being averaged (e.g., $|S|$ for member-level quantities and $\binom{|S|}{2}$ for pairwise
quantities). We define:
\begin{equation}
    z(u(S)) = \frac{u(S)-\mu_u}{\sigma_u/\sqrt{n_u(S)}},
    \label{eq:z_score}
\end{equation}
where $(\mu_u,\sigma_u)$ are computed over the corresponding base items in the candidate pool: the
set $\{q_i^t\}$ for quality, the set $\{1-\bar{R}_{ij}\}$ for error decorrelation, and the set
$\{\bar{J}_{ij}\}$ for distributional divergence.

The \cmas{} selection score for task $t$ can then be formulated as:
\begin{equation}
    \mathrm{score}_t(S)
    = z\!\left(\frac{1}{|S|}\sum_{i\in S} q_i^t\right)
    + \lambda_1\, z\!\left(\mathrm{HI}_{err}(S;\bar{R})\right)
    + \lambda_2\, z\!\left(\mathrm{HI}_{dist}(S;\bar{J})\right),
    \label{eq:objective}
\end{equation}
\noindent\begin{minipage}[t]{\textwidth}
\setlength{\intextsep}{6pt}
\setlength{\parskip}{6pt}
\begin{wrapfloat}{algorithm}{r}{0.55\textwidth}
    \small
    \captionsetup{justification=raggedright,singlelinecheck=false}
    \caption{Multi-start forward greedy selection for~task~$t$}
    \label{alg:greedy}
    \begin{algorithmic}[1]
        \STATE {\bfseries Input:} Candidates $\mathcal{M}$, team size $k$, qualities $\{q_i^t\}$,
        pooled matrices $\bar{R},\bar{J}$, weights $(\lambda_1,\lambda_2)$
        \STATE {\bfseries Output:} Selected team $S_t^*$
        \STATE $S_t^* \leftarrow \emptyset$, $\mathrm{best} \leftarrow -\infty$
        \FOR{each seed $s\in\mathcal{M}$}
            \STATE $S\leftarrow\{s\}$
            \WHILE{$|S|<k$}
                \STATE $j^* \leftarrow \argmax_{j\in\mathcal{M}\setminus S}\; \mathrm{score}_t(S\cup\{j\})$
                \STATE $S \leftarrow S\cup\{j^*\}$
            \ENDWHILE
            \IF{$\mathrm{score}_t(S)>\mathrm{best}$}
                \STATE $\mathrm{best}\leftarrow \mathrm{score}_t(S)$, $S_t^*\leftarrow S$
            \ENDIF
        \ENDFOR
        \STATE \textbf{return} $S_t^*$
    \end{algorithmic}
\end{wrapfloat}

where $\lambda_1,\lambda_2\ge 0$ are trade-off coefficients for
$\mathrm{HI}_{err}$ and $\mathrm{HI}_{dist}$.
The selection objective is $\max_{|S|=k}\mathrm{score}_t(S)$.

To efficiently solve this combinatorial optimization problem, we adopt a multi-start forward greedy procedure: starting from a seed model, we iteratively add the candidate that yields the largest score improvement until reaching size $k$, and we repeat this process for each seed in $\mathcal{M}$, returning the best team found.
The full procedure is summarized in \mbox{Algorithm~\ref{alg:greedy}}.

\subsection{Aggregation}
\label{sec:aggregation}
\addtocounter{WF@wrappedlines}{-3}
\vspace{6pt}

Our method is agnostic to the downstream aggregator, as team selection relies solely on capability
and complementarity statistics obtained through offline profiling. To examine whether the gains
from team selection persist across aggregation mechanisms, we evaluate the same selected teams
using \textsc{Choice-Soft}, \textsc{PoE}, \textsc{DS}, and \textsc{Stacking}. \textsc{Choice-Soft} and
\textsc{PoE} use fixed rules based on the mean and normalized product of member probabilities,
respectively, whereas \textsc{DS} and \textsc{Stacking} learn aggregation parameters from development
data to account for members' predictive characteristics and reliability differences. Full definitions
and training details are provided in \Cref{app:aggregation_rules}.

\finishwrap
\end{minipage}

\subsection{Theoretical Insights}
\label{sec:theoretical_insights}

We analyze how correctness dependence and choice-distribution divergence jointly affect team
accuracy, and when their benefits compensate for a reduction in individual model quality.

\paragraph{A choice-probability model.}
Fix a task distribution and three-member teams on four choices. With the true answer listed first,
a correct member outputs $p^{\rm c}=(a,u,u,u)$; an incorrect member outputs
$p^w=(a,b,c,c)$ with the last three coordinates permuted to place $b$ on a wrong
answer $w$. Assume
\begin{equation}
    0<c<a<b,\quad a+b+2c=1,\quad
    u=(b+2c)/3,\quad a>(2b+u)/3.
    \label{eq:joint_model}
\end{equation}
Members' correctness and preferred wrong answers may be arbitrarily dependent.
For a team $S$, write $\bar q_S=\frac13\sum_{i\in S}\Pr(C_i=1)$,
$e_S=1-\bar q_S$, and let $D_S$ be the mean pairwise probability of both members
being incorrect. Let $K_S$ be the mean probability of both being incorrect
\emph{and preferring the same wrong answer}, and let $J_S$ be their mean JSD on
this task. Define $d=\JSD(p^w,p^{w'})$ for $w\ne w'$, $h=\JSD(p^{\rm c},p^w)$,
and $\theta=2h/d$.

\begin{proposition}[Shared errors and wrong-answer concentration]
\label{prop:joint_failure}
Under \eqref{eq:joint_model}, $0<\theta<1$ and
\begin{equation}
    K_S=\theta e_S+(1-\theta)D_S-\frac{J_S}{d}.
    \label{eq:joint_collision}
\end{equation}
For either \textnormal{\textsc{Choice-Soft}} or \textnormal{\textsc{PoE}},
$\Pr(\hat y_S\ne Y)\le K_S$. Both aggregators fail exactly when all three
members prefer the same wrong answer.
\end{proposition}

\begin{remark}[Complementarity of the two signals]
\label{rem:correctness_information}
At fixed marginal accuracies, Yule's $Q$ increases strictly with pairwise
double-fault probability (\Cref{prop:yule_q_monotonicity}). Correctness dependence
therefore characterizes shared errors, while choice-space JSD supplies information
about their concentration on particular answers. Equation~\eqref{eq:joint_collision}
connects both signals to a source of collective failure beyond individual quality.
\end{remark}

\paragraph{Accuracy gains over quality-only selection.}
Let $S_Q$ be the size-three Quality-Only team, and abbreviate its quantities
by a subscript $Q$. Define the baseline correction
$\varepsilon_Q=\min\{K_Q,\,e_Q-(D_Q+K_Q)/2\}\ge0$.

\begin{proposition}[A sufficient condition for improvement]
\label{prop:quality_only_gain}
Under \eqref{eq:joint_model}, the team $S$ returned by Algorithm~\ref{alg:greedy},
evaluated with the same \textnormal{\textsc{Choice-Soft}} or \textnormal{\textsc{PoE}}
aggregator as $S_Q$, satisfies
\begin{equation}
    \operatorname{Acc}(S)-\operatorname{Acc}(S_Q)
    \ge (1-\theta)(D_Q-D_S)+\frac{J_S-J_Q}{d}
    -\theta(\bar q_Q-\bar q_S)-\varepsilon_Q.
    \label{eq:quality_only_gain}
\end{equation}
A positive right-hand side guarantees strictly higher team accuracy.
\end{proposition}

\begin{remark}[Quality--heterogeneity trade-off]
\label{rem:quality_complementarity}
Equation~\eqref{eq:quality_only_gain} shows that a team with lower mean member accuracy
can still outperform Quality-Only. A sufficient condition is that the weighted gains
from reduced shared errors and increased choice-distribution divergence exceed the
quality penalty and the baseline correction $\varepsilon_Q$. For a baseline whose
members make identical predictions, $\varepsilon_Q=0$, so these gains need only outweigh
the quality penalty. \Cref{app:joint_failure,app:quality_gain} give complete proofs;
\Cref{app:selection_gain} connects the condition to the standardized score,
pooled profiles, and the terminal teams generated by multi-start search.
\end{remark}

\section{Experiments}
\label{sec:experiments}

We evaluate \cmas{} on various benchmarks to address the following research questions:

\noindent\textbf{RQ1 (Effectiveness and Generalization).}
Does heterogeneity-aware team selection yield consistent gains over quality-only baselines, both on the primary task set and on out-of-distribution (OOD) benchmarks?

\noindent\textbf{RQ2 (Mechanism).}
How do error decorrelation ($\mathrm{HI}_{err}$) and distributional divergence ($\mathrm{HI}_{dist}$) individually contribute to team performance, and what underlies these gains?

\noindent\textbf{RQ3 (Robustness).}
Is \cmas{} robust to the choice of heterogeneity weights $(\lambda_1,\lambda_2)$ and to limited profiling data?

\subsection{Setup}
\label{sec:exp_setup}

\paragraph{Tasks and Benchmarks.}
We evaluate \cmas{} on 7 primary datasets and test its OOD generalization on 6 held-out datasets. The 13 public multiple-choice benchmarks span four domains: (1) \emph{math reasoning}, AQUA-RAT~\citep{ling2017program} and MathQA~\citep{amini2019mathqa}; (2) \emph{natural science}, ARC-Challenge~\citep{clark2018think} and OpenBookQA~\citep{mihaylov2018can}; (3) \emph{reading comprehension and logical reasoning}, RACE~\citep{lai2017race}, ReClor~\citep{yu2020reclor}, and LogiQA2~\citep{liu2023logiqa}; and (4) \emph{general or domain knowledge}, MMLU~\citep{hendrycks2021measuring}, MMLU-Pro~\citep{wang2024mmlu}, C-EVAL~\citep{huang2023c}, CommonsenseQA~\citep{talmor2019commonsenseqa}, MedQA~\citep{jin2021medqa}, and StrategyQA~\citep{geva2021did}. For efficiency, we randomly subsample both the dev and test splits to 1,068 examples each.

\paragraph{Models.}
We use $m=11$ open-weight chat models in the 7--9B parameter class for the candidate pool. The pool includes
Llama-3.1-8B-Instruct~\citep{grattafiori2024llama}, Qwen3-8B~\citep{yang2025qwen3}, Gemma-2-9B-IT~\citep{team2024gemma}, Ministral-3-8B-Instruct~\citep{liu2026ministral}, Granite-3.3-8B-Instruct~\citep{granite2024granite},
GLM-4-9B-Chat~\citep{glm2024chatglm}, Nemotron-Nano-9B-v2~\citep{basant2025nvidia}, EuroLLM-9B-Instruct~\citep{martins2025eurollm}, OLMo-3-7B-Instruct~\citep{olmo2025olmo}, Apertus-8B-Instruct~\citep{apertus2025apertus},
and InternLM3-8B-Instruct~\citep{cai2024internlm2}.

\paragraph{Baselines.}
We compare:
(1) \emph{Self-Consistency} \citep{wang2023self}: the best single model on the dev set sampled $k=3$ times with stochastic decoding and aggregated by majority vote;
(2) \emph{Random-k}: uniformly sample $k=3$ models without replacement and report mean accuracy over 100 samples;
(3) \emph{Caruana} \citep{caruana2004ensemble}: forward ensemble selection on $D_t^{\mathrm{dev}}$ with replacement (allow repeats to express non-uniform member frequencies);
(4) \emph{Quality-Only}: select $k=3$ by maximizing mean development accuracy (i.e., $\lambda_1=\lambda_2=0$);
and (5) \emph{\cmas{}}: our heterogeneity-aware selection with $k=3$.

\paragraph{Implementation details.}
We select teams by maximizing the standardized objective in \Cref{eq:objective} with weights $(\lambda_1,\lambda_2)=(0.13, 0.05)$. For aggregation, DS uses additive smoothing with $\varepsilon=10^{-3}$ when estimating the class prior and confusion matrices; Stacking trains a linear combiner on $D_t^{\mathrm{dev}}$ and evaluates on $D_t^{\mathrm{test}}$ ($L_2$ coefficient $\beta=0.003$); the other aggregators are parameter-free. For statistical testing, we use paired bootstrap with 2,000 resamples to compute confidence intervals (CI) for accuracy differences on the test set (paired by instance).

\begin{table}[t]
    \centering
    \caption{Test accuracy (\%) on the 7 primary benchmarks; teams of size $k=3$
are selected from $m=11$ candidates. $\uparrow$: gain over Random-$k$ (pp);
\textbf{bold}: best within each aggregation block. Held-out and full results:
\Cref{tab:ood_results,tab:complete_results}.
ARC-C: ARC-Challenge; CSQA: CommonsenseQA; OBQA: OpenBookQA.}
    \label{tab:main_results}
    \begingroup
    \footnotesize
    \setlength{\tabcolsep}{0.8pt}
    \renewcommand{\arraystretch}{1.08}
    \begin{tabular*}{\linewidth}{@{\extracolsep{\fill}}cl*{8}{l}@{}}
        \toprule[1.25pt]
        \multicolumn{2}{c}{\multirow[c]{2}{*}{\textbf{Method}}}
        & \multicolumn{7}{c}{\textbf{Dataset}}
        & \multirow[c]{2}{*}{\textbf{Avg.}} \\
        \multicolumn{2}{c}{} 
        & \textbf{ARC-C} & \textbf{CSQA} & \textbf{LogiQA2} & \textbf{MedQA} & \textbf{MMLU} & \textbf{MMLU-Pro} & \textbf{OBQA} & \\

        \midrule[1.1pt]
        \rowcolor[rgb]{0.93,0.93,0.93}
        \multicolumn{10}{c}{\textbf{Closed-Source Models (for reference)}} \\
        \multicolumn{2}{l}{Gemini-2.5-Flash} & $92.70$ & $81.65$ & $60.96$ & $77.43$ & $81.37$ & $58.61$ & $92.70$ & $77.92$ \\
        \multicolumn{2}{l}{GPT-4o} & $94.10$ & $83.33$ & $64.98$ & $84.74$ & $84.64$ & $50.19$ & $91.85$ & $79.12$ \\
        \midrule
        \rowcolor[rgb]{0.93,0.93,0.93}
        \multicolumn{10}{c}{\textbf{Single-Model Aggregation}} \\
        \multicolumn{2}{l}{Self-Consistency} & $89.14$ & $78.65$ & $67.32$ & $61.70$ & $73.41$ & $44.76$ & $88.58$ & $71.94$ \\
        \midrule
        \rowcolor[rgb]{0.93,0.93,0.93}
        \multicolumn{10}{c}{\textbf{Team Selection Baselines}} \\
	        \multirow[c]{4}{*}{\rotatebox[origin=c]{90}{\textit{Choice-Soft}}} & Random-$k$ & $87.61$ & $78.87$ & $62.49$ & $60.50$ & $72.08$ & $43.15$ & $85.08$ & $69.97$ \\
	        & Caruana & $89.98$ & $81.65$ & $70.51$ & $63.11$ & $73.69$ & $48.22$ & $\mathbf{90.54}$ & $73.96$ \\
	        & Quality-Only & $\mathbf{91.48}$ & $81.46$ & $\mathbf{70.88}$ & $\mathbf{66.39}$ & $\mathbf{76.59}$ & $47.00$ & $89.70$ & $74.79$ \\
	        & \textbf{\cmas{} (Ours)} & \scoregain{$\mathbf{91.48}$}{3.87} & \scoregain{$\mathbf{83.90}$}{5.02} & \scoregain{$\mathbf{70.88}$}{8.39} & \scoregain{$\mathbf{66.39}$}{5.89} & \scoregain{$\mathbf{76.59}$}{4.51} & \scoregain{$\mathbf{48.50}$}{5.35} & \scoregain{$90.45$}{5.37} & \scoregain{$\mathbf{75.45}$}{5.48} \\
	        \midrule
	        \multirow[c]{4}{*}{\rotatebox[origin=c]{90}{\textit{PoE}}} & Random-$k$ & $87.58$ & $79.20$ & $63.04$ & $60.81$ & $72.30$ & $43.41$ & $84.79$ & $70.16$ \\
	        & Caruana & $90.54$ & $83.05$ & $\mathbf{71.44}$ & $65.07$ & $74.06$ & $\mathbf{49.34}$ & $\mathbf{90.82}$ & $74.91$ \\
	        & Quality-Only & $\mathbf{91.57}$ & $82.68$ & $70.69$ & $\mathbf{65.36}$ & $\mathbf{76.69}$ & $48.69$ & $89.33$ & $75.00$ \\
	        & \textbf{\cmas{} (Ours)} & \scoregain{$\mathbf{91.57}$}{3.99} & \scoregain{$\mathbf{83.90}$}{4.70} & \scoregain{$70.69$}{7.65} & \scoregain{$\mathbf{65.36}$}{4.55} & \scoregain{$\mathbf{76.69}$}{4.39} & \scoregain{$48.97$}{5.56} & \scoregain{$89.51$}{4.72} & \scoregain{$\mathbf{75.24}$}{5.08} \\
	        \midrule
	        \multirow[c]{4}{*}{\rotatebox[origin=c]{90}{\textit{DS}}} & Random-$k$ & $87.56$ & $78.00$ & $62.90$ & $60.87$ & $71.52$ & $42.46$ & $84.96$ & $69.75$ \\
	        & Caruana & $86.24$ & $79.49$ & $68.82$ & $63.95$ & $72.94$ & $\mathbf{47.47}$ & $89.61$ & $72.65$ \\
	        & Quality-Only & $\mathbf{90.92}$ & $81.18$ & $\mathbf{70.32}$ & $\mathbf{65.17}$ & $\mathbf{75.47}$ & $47.28$ & $89.51$ & $74.26$ \\
	        & \textbf{\cmas{} (Ours)} & \scoregain{$\mathbf{90.92}$}{3.36} & \scoregain{$\mathbf{82.87}$}{4.87} & \scoregain{$\mathbf{70.32}$}{7.42} & \scoregain{$\mathbf{65.17}$}{4.30} & \scoregain{$\mathbf{75.47}$}{3.95} & \scoregain{$47.10$}{4.64} & \scoregain{$\mathbf{90.45}$}{5.49} & \scoregain{$\mathbf{74.61}$}{4.86} \\
	        \midrule
	        \multirow[c]{4}{*}{\rotatebox[origin=c]{90}{\textit{Stacking}}} & Random-$k$ & $89.30$ & $80.20$ & $65.54$ & $62.82$ & $73.55$ & $44.82$ & $87.35$ & $71.94$ \\
	        & Caruana & $90.54$ & $82.77$ & $\mathbf{70.60}$ & $64.42$ & $73.88$ & $48.78$ & $\mathbf{91.57}$ & $74.65$ \\
	        & Quality-Only & $\mathbf{91.67}$ & $81.74$ & $\mathbf{70.60}$ & $\mathbf{65.54}$ & $\mathbf{76.12}$ & $47.19$ & $90.17$ & $74.72$ \\
	        & \textbf{\cmas{} (Ours)} & \scoregain{$\mathbf{91.67}$}{2.37} & \scoregain{$\mathbf{84.18}$}{3.98} & \scoregain{$\mathbf{70.60}$}{5.06} & \scoregain{$\mathbf{65.54}$}{2.72} & \scoregain{$\mathbf{76.12}$}{2.57} & \scoregain{$\mathbf{48.97}$}{4.15} & \scoregain{$\mathbf{91.57}$}{4.22} & \scoregain{$\mathbf{75.52}$}{3.58} \\
	        \bottomrule[1.25pt]
	    \end{tabular*}
    \endgroup
\end{table}

\raggedbottom
\subsection{Effectiveness of \cmas{} (RQ1)}
\label{sec:main_results}

We report primary results on 7 datasets in Table~\ref{tab:main_results} and OOD generalization on the remaining 6 in \Cref{tab:ood_results}.

\paragraph{Primary Results.}
Under strict development-test separation, \cmas{} delivers small but consistent improvements over the Quality-Only baseline. As shown in Table~\ref{tab:main_results}, \cmas{} reaches 75.52\% with the \textsc{Stacking} aggregator, outperforming Quality-Only (74.72\%) by $+0.80$\,pp (95\% CI: [0.44, 1.16]), while also surpassing both Self-Consistency and Random-$k$ by $+3.58$\,pp.

\paragraph{OOD Generalization.}
To assess robustness under distribution shifts, we evaluate \cmas{} on 6 held-out benchmarks. \cmas{} achieves 66.74\% average accuracy, surpassing Quality-Only (66.43\%, +0.31 pp) and significantly outperforming Self-Consistency (64.45\%, $+2.29$\,pp) and Random-$k$ (61.12\%, $+5.62$\,pp), confirming that heterogeneity signals capture intrinsic complementarity that transfers beyond the primary task set.

\paragraph{Non-regression and Consistency.}
Beyond average gains, safety is critical for practical deployment. Under \textsc{Stacking}, the wins come from the three benchmarks where \cmas{} selects a different team from Quality-Only (CommonsenseQA, MMLU-Pro, and OpenBookQA), with a mean gain of $+1.87$\,pp. On the remaining four benchmarks the teams are identical and predictions match Quality-Only by construction (\Cref{app:team_composition}). In each differing case, \cmas{} replaces a redundant high-accuracy member with a more complementary one rather than trading quality for diversity. \cmas{} thus acts as a \emph{safe plugin}: it preserves the strongest baseline's performance and adds gains where complementarity helps. The pattern holds across all four aggregators, with the largest margin under Stacking where the learned combiner benefits most from heterogeneous inputs.

\subsection{Mechanism Analysis: Why Heterogeneity Matters? (RQ2)}
\label{sec:mechanism}

While the main results confirm the effectiveness of our approach, we further investigate the source of these gains through quantitative ablation studies (\Cref{tab:ablation}) and qualitative visualization (\Cref{fig:mds}).

\begin{figure}[!t]
    \centering
    \begin{subfigure}[t]{0.32\linewidth}
        \centering
        \includegraphics[width=\linewidth]{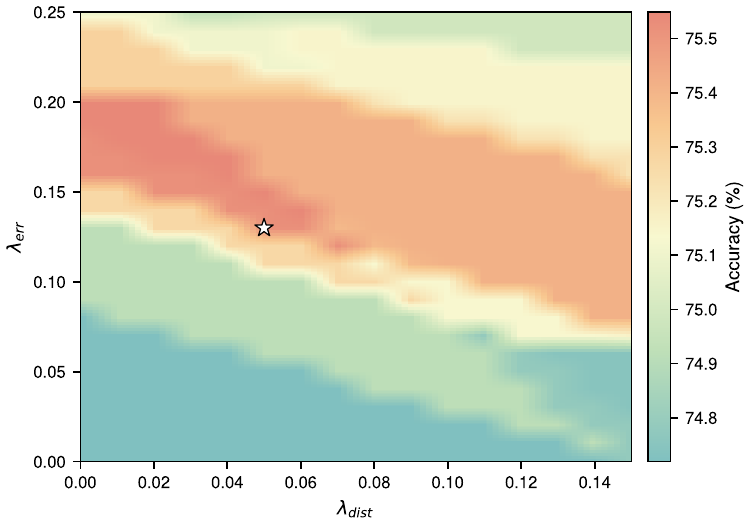}
        \caption{Sensitivity of\\heterogeneity weights.}
        \label{fig:lambda_sensitivity}
    \end{subfigure}
    \hfill
    \begin{subfigure}[t]{0.32\linewidth}
        \centering
        \includegraphics[width=\linewidth]{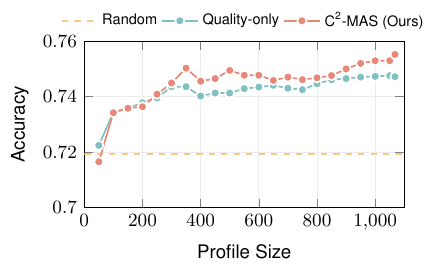}
        \caption{Accuracy vs.\ profiling size.}
        \label{fig:dev_size_performance}
    \end{subfigure}
    \hfill
    \begin{subfigure}[t]{0.32\linewidth}
        \centering
        \includegraphics[width=\linewidth]{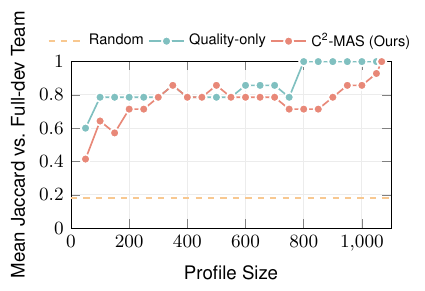}
        \caption{Selection stability (Jaccard vs.\ full-dev team).}
        \label{fig:dev_size_stability}
    \end{subfigure}
    \caption{
    Robustness and data efficiency of \cmas{}.
    (a) Test accuracy under \textsc{Stacking} across heterogeneity weights; $(0,0)$ denotes Quality-Only and $\star$ marks the configuration used in our main experiments.
    (b--c) Test accuracy and selection stability versus profiling size, with stability measured by mean Jaccard similarity to the full-dev team.
    }
    \label{fig:sensitivity}
\end{figure}

\begingroup
\setlength{\intextsep}{6pt}
\begin{wraptable}{r}{0.55\textwidth}
    \centering
    \captionsetup{justification=raggedright,singlelinecheck=false,skip=2pt}
    \begingroup
    \small
    \setlength{\tabcolsep}{1.7pt}
    \renewcommand{\arraystretch}{1.0}
      \begin{tabular}{@{}c@{\hspace{3pt}}c@{\hspace{5pt}}|ccccc@{}}
        \toprule
        $\HIerr$ & $\HIdist$ & \makecell{\textsc{Choice-}\\\textsc{Soft}} & \textsc{PoE} & \textsc{DS} & \textsc{Stacking} & \textsc{Overall} \\
        \midrule
        \xmark & \xmark & 74.79 & 75.00 & 74.26 & 74.72 & 74.69 \\
        \checkmark & \xmark & 75.24 & 75.20 & \textbf{74.64} & 75.27 & 75.09 \\
        \checkmark & \checkmark & \textbf{75.45} & \textbf{75.24} & 74.61 & \textbf{75.52} & \textbf{75.21} \\
        \bottomrule
      \end{tabular}
    \endgroup
    \caption{Ablation study of \cmas{}. The first row corresponds to the Quality-Only baseline.}
    \label{tab:ablation}
\end{wraptable}

\paragraph{Contribution of Profiling Signals.}\mbox{}\\
\mbox{Table~\ref{tab:ablation}} presents the performance impact of removing individual heterogeneity signals.
(1) \emph{Error Decorrelation}: With $\mathrm{HI}_{err}$, all four aggregators improve over Quality-Only, supporting the value of reducing shared errors. (2) \emph{Distributional Divergence}: Adding $\mathrm{HI}_{dist}$ further improves overall accuracy, indicating that choice probabilities contain complementary information. (3) \emph{Synergy}: The combined signals achieve the highest overall score, consistent with their distinct roles in profiling complementarity.

\finishwrap
\endgroup

\noindent\begin{minipage}[t]{\textwidth}
\setlength{\intextsep}{6pt}
\begin{wrapfigure}{r}{0.42\textwidth}
    \centering
    \captionsetup{justification=raggedright,singlelinecheck=false}
    \includegraphics[width=\linewidth]{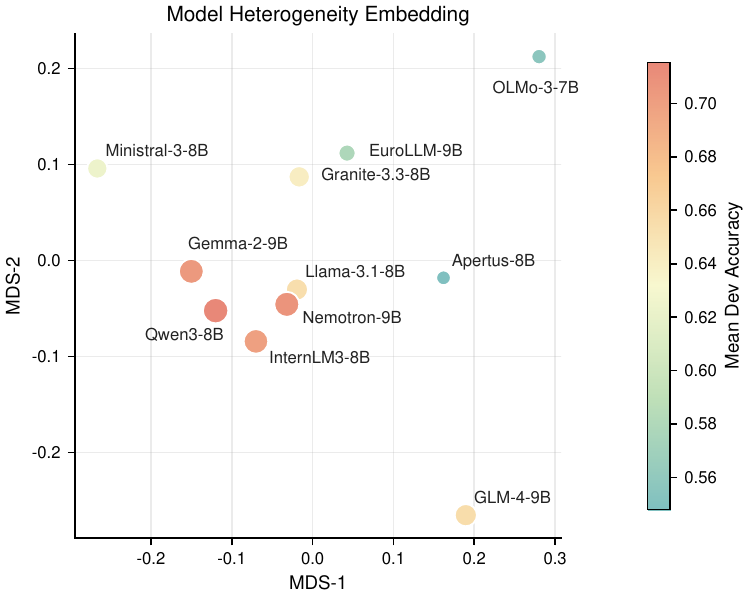}
    \caption{MDS visualization of model heterogeneity.}
    \label{fig:mds}
\end{wrapfigure}

\paragraph{Visualization of the Heterogeneity Landscape.}
To examine the limitations of Quality-Only selection, we visualize model heterogeneity using multidimensional scaling (MDS). The embedding uses pairwise Yule's $Q$ and choice-space JSD averaged across the seven primary benchmarks. Models with similar prediction patterns cluster together. As shown in Figure~\ref{fig:mds}, high-accuracy models (e.g., Qwen3, Gemma-2) tend to cluster tightly in the center. This reveals a key phenomenon: the strongest models are often the most similar, likely sharing knowledge boundaries and blind spots. Consequently, a Quality-Only strategy that simply selects the top-$k$ models often results in a highly redundant team. In contrast, distinct models with slightly lower accuracy (e.g., GLM-4) are scattered along the periphery. \cmas{} excels by identifying and recruiting these ``outlier'' members while maintaining a quality baseline. These diverse perspectives fill the blind spots of strong models, thereby breaking the performance ceiling of homogeneous teams.

\finishwrap
\end{minipage}

\Needspace{8\baselineskip}
\subsection{Robustness and Efficiency (RQ3)}
\label{sec:robustness}

For real-world deployment, sensitivity to hyperparameters and dependence on cold-start data are key concerns. We verify the robustness of \cmas{} as follows.

\paragraph{Sensitivity to Weights.}
Figure~\ref{fig:sensitivity}(a) shows broad regions of high accuracy away from the Quality-Only origin $(0,0)$. The selection score is affine in $(\lambda_1,\lambda_2)$, so the greedy output is constant within regions where its comparison signs remain fixed (\Cref{app:selection_stability}). This explains why nearby weights can select the same team. The observed high-accuracy regions indicate low sensitivity to local weight changes in the evaluated range.

\paragraph{Data Efficiency and Stability.}
Figures~\ref{fig:sensitivity}(b) and (c) examine profiling data size. \cmas{} retains its advantage with a few hundred development instances, while Jaccard similarity to the team selected using the full dev set rises rapidly and saturates early. Together, these results show that reliable team composition and competitive accuracy can be achieved with limited development data, supporting the practical efficiency of our framework.

\section{Conclusion}
\label{sec:conclusion}

We proposed \cmas{}, a heterogeneity-driven framework for selecting complementary LLM teams under budget constraints. By explicitly modeling \textbf{error decorrelation} and \textbf{distributional divergence}, our approach quantifies inter-member complementarity beyond individual capability. Extensive experiments validate the effectiveness, robustness, and efficiency of \cmas{}, demonstrating that it consistently outperforms quality-only baselines by recruiting complementary models to fill the blind spots of top-performing candidates. Future work could extend heterogeneity profiling to open-ended generation and explore how complementary teams collaborate through multi-turn debate and iterative verification.

\FloatBarrier
\label{page:main_end}

\section*{Ethics Statement}
Our experiments use only publicly available multiple-choice benchmarks under their original licenses, and we evaluate open-weight models alongside two closed-source APIs (GPT-4o, Gemini-2.5-Flash) accessed through their official endpoints in accordance with the providers' terms of use. No human subjects, crowdsourced annotation, or personally identifiable information are involved.

A risk specific to aggregation-based methods is that biases or errors present in individual models may be amplified, attenuated, or obscured when their outputs are combined. Although our heterogeneity metrics are designed to diversify error patterns, they do not directly target social bias, and practitioners deploying \cmas{} should pair team selection with bias auditing of the resulting predictions.

\section*{Reproducibility Statement}
The candidate models, benchmarks, selection weights, and aggregation settings are described in \Cref{sec:exp_setup,sec:method}. The appendix provides the choice-scoring implementation, paired-bootstrap procedure, hyperparameter selection protocol, and complete results (\Cref{app:implementation,app:complete_results}).

\Needspace{10\baselineskip}
\section*{Authors and Affiliations}
\begingroup
\small\setlength{\parindent}{0pt}\setlength{\parskip}{0pt}
\raggedright
\paperauthors\par
\vspace{9pt}
\renewcommand{\arraystretch}{1.12}
\begin{tabularx}{\linewidth}{@{}r@{\hspace{0.65em}}>{\raggedright\arraybackslash}X@{}}
\textsuperscript{1} & College of Intelligent Robotics and Advanced Manufacturing, Fudan University, Shanghai, China\\[4pt]
\textsuperscript{2} & Tongji University, Shanghai, China\\[4pt]
\textsuperscript{3} & College of Future Information Technology, Fudan University, Shanghai, China
\end{tabularx}\par
\endgroup

\bibliography{references}
\bibliographystyle{plainnat}

\clearpage
\appendix
\pdfbookmark[0]{Appendix Contents}{appendix.contents}
\section*{Appendix Contents}
\label{app:contents}

\begingroup
\makeatletter
\c@tocdepth=2\relax
\parskip=0pt
\newcommand{\appendixcontentsline}[2]{%
  \ifnum#1=1\relax
    \addvspace{0.5\baselineskip}%
    \@dottedtocline{1}{0em}{2.2em}%
      {\hyperref[#2]{\bfseries\numberline{\ref*{#2}}\nameref*{#2}}}%
      {\hyperref[#2]{\bfseries\pageref*{#2}}}%
  \else
    \@dottedtocline{2}{1.5em}{2.8em}%
      {\hyperref[#2]{\numberline{\ref*{#2}}\nameref*{#2}}}%
      {\hyperref[#2]{\pageref*{#2}}}%
  \fi
}

\appendixcontentsline{1}{app:notations}
\appendixcontentsline{1}{app:extended_related_work}
\appendixcontentsline{2}{app:rw_multi_agent}
\appendixcontentsline{2}{app:rw_multi_model}
\appendixcontentsline{1}{app:theory}
\appendixcontentsline{2}{app:correctness_dependence}
\appendixcontentsline{2}{app:joint_failure}
\appendixcontentsline{2}{app:quality_gain}
\appendixcontentsline{2}{app:soft_recovery}
\appendixcontentsline{2}{app:selection_stability}
\appendixcontentsline{2}{app:pooling_geometry}
\appendixcontentsline{2}{app:selection_gain}
\appendixcontentsline{2}{app:stacking_theory}
\appendixcontentsline{1}{app:implementation}
\appendixcontentsline{2}{app:scoring_implementation}
\appendixcontentsline{2}{app:bootstrap}
\appendixcontentsline{2}{app:hyperparameter_selection}
\appendixcontentsline{2}{app:selection_cost}
\appendixcontentsline{2}{app:aggregation_rules}
\appendixcontentsline{1}{app:additional_experiments}
\appendixcontentsline{2}{app:complete_results}
\appendixcontentsline{2}{app:team_composition}
\appendixcontentsline{2}{app:robustness_ablation}
\appendixcontentsline{2}{app:pooled_stability_results}
\appendixcontentsline{1}{app:limitations}

\makeatother
\endgroup

\clearpage
\let\appendixsection\section
\renewcommand{\section}{\FloatBarrier\appendixsection}
\let\appendixsubsection\subsection
\renewcommand{\subsection}{\FloatBarrier\appendixsubsection}

\section{Notations}
\label{app:notations}

\begingroup
\small
\begin{longtable}{p{0.29\linewidth}p{\dimexpr0.71\linewidth-4\tabcolsep\relax}}
\caption{Notations and Definitions}\label{tab:notations}\\
\toprule
\textbf{Notation} & \textbf{Definition} \\
\midrule
\endfirsthead
\multicolumn{2}{l}{\textit{Table \thetable{} continued.}}\\
\toprule
\textbf{Notation} & \textbf{Definition} \\
\midrule
\endhead
\midrule
\multicolumn{2}{r}{\textit{Continued on next page.}}\\
\endfoot
\bottomrule
\endlastfoot

    $\mathcal{T}$ & Set of tasks used for cross-task pooling. \\
    $t$ & Task / benchmark index. \\
    $D_t^{\mathrm{dev}},\, D_t^{\mathrm{test}}$ & Development and test splits for task $t$ (strict dev$\rightarrow$test). \\
    $\mathcal{Y}_t$ & Label / option set for task $t$ (multiple-choice). \\
    $x,\,y$ & An instance and its gold label ($y\in\mathcal{Y}_t$). \\
    
    $\mathcal{M}=\{1,\ldots,m\}$ & Index set of the $m$ candidate agents (callable LLMs). \\
    $m$ & Number of candidate agents. \\
    $k$ & Team size budget. \\
    $S\subseteq\mathcal{M}$ & A team of agents, with $|S|=k$. \\
    $i\in\mathcal{M}$ & Index of a candidate agent. \\
    
    $p_i(\cdot\mid x)$ & Agent $i$'s aligned choice distribution over $\mathcal{Y}_t$ for instance $x$ (task index omitted when clear). \\
    $\hat{y}_i(x)$ & Predicted label of agent $i$, $\hat{y}_i(x)=\arg\max_{y\in\mathcal{Y}_t}p_i(y\mid x)$. \\
    $C_i(x)\in\{0,1\}$ & Correctness indicator of agent $i$ on example $x$. \\
    $N_{ab}(i,j)$ & Count of examples where $(C_i(x),C_j(x))=(a,b)$ for $a,b\in\{0,1\}$. \\
    $R(i,j)$ & Error-correlation / complementarity statistic between agents $i$ and $j$ (e.g., Yule's $Q$). \\
    $q_i^t$ & Estimated individual quality of agent $i$ on task $t$ (development accuracy). \\
    $\bar{R}$ & Pooled Yule's $Q$ association matrix estimated from development data. \\
    $\bar{J}$ & Pooled distribution-divergence matrix estimated from development data (across tasks if applicable). \\
    $\mathrm{HI}_{\mathrm{err}}(S;\bar{R})$ & Error-level heterogeneity / complementarity of team $S$ (computed from $\bar{R}$). \\
    $\mathrm{HI}_{\mathrm{dist}}(S;\bar{J})$ & Distribution-level heterogeneity / disagreement of team $S$ (computed from $\bar{J}$). \\
    $z(\cdot)$ & Standard-error-scaled normalization of a team average using candidate-pool base-item statistics. \\
    $\lambda_1,\,\lambda_2$ & Objective weights trading off capability and complementarity. \\
    $\mathrm{score}_t(S)$ & Standardized selection objective for task $t$ (see \Cref{eq:objective}). \\
    $S_t^{*}$ & Size-$k$ team returned by multi-start greedy search on $\mathrm{score}_t(S)$. \\
    $p_{\textsc{cs}}^t(y\mid x)$ & Choice-Soft aggregation distribution (mean of members' $p_i^t$). \\
    $p_{\textsc{poe}}^t(y\mid x)$ & PoE aggregation distribution (product of members' $p_i^t$, renormalized). \\
    $\pi^t(y)$ & Class prior estimated from $D_t^{\mathrm{dev}}$ (used in DS). \\
    $A_i^t(a,b)$ & DS confusion matrix entry estimating $\mathbb{P}(\hat{y}_i=b\mid y=a)$ on $D_t^{\mathrm{dev}}$. \\
    $\varepsilon$ & Additive smoothing constant for DS priors/confusions (e.g., $\varepsilon=10^{-3}$). \\
    $p_{\textsc{ds}}^t(y\mid x)$ & DS posterior (up to proportionality) combining $\pi^t$ and $\{A_i^t\}$. \\
    $\phi(x)$ & Stacking feature vector formed by concatenating members' $p_i^t(\cdot\mid x)$. \\
    $W,\,b$ & Linear stacking combiner parameters. \\
    $\beta$ & $L_2$ regularization coefficient for stacking. \\
    $\pi_{W,b}(y\mid x)$ & Stacking aggregation distribution, $\pi_{W,b}(y\mid x)=\mathrm{softmax}(W^\top\phi(x)+b)_y$. \\
    $\hat{y}_{S_t}(x)$ & Final team prediction on $x$ under a chosen aggregator. \\
    $\bar q_S,\,e_S$ & Mean member accuracy and error probability in the fixed-task theoretical model. \\
    $D_S,\,K_S,\,J_S$ & Mean pairwise double-fault probability, same-wrong-answer probability, and JSD in that model. \\
    $d,\,h,\,\theta$ & Divergence between distinct wrong-answer templates, divergence between correct and incorrect templates, and $\theta=2h/d$. \\
    $\varepsilon_Q$ & Baseline correction in the accuracy comparison with Quality-Only. \\
    \end{longtable}
\endgroup

\section{Additional Related Works}
\label{app:extended_related_work}

\subsection{Multi-Agent LLM Systems}
\label{app:rw_multi_agent}

Much prior work builds multi-agent collaboration on predefined roles and fixed workflows, ranging from role-playing paradigms \citep{li2023camel, hong2024metagpt} and workflow-based dialogue orchestration \citep{wu2024autogen, qian2024chatdev, chen2024agentverse} to multi-agent simulation environments for modeling collective and social behaviors \citep{park2023generative, zhou2024sotopia, tang2025gensim, piao2025agentsociety}. To improve reliability, studies model interaction as structured group decision processes such as debate \citep{du2024improving, liang2024encouraging}, vote-based aggregation \citep{choi2025debate}, and consensus mechanisms \citep{chen2024reconcile}, and analyze protocol interventions to understand mechanisms and failure modes \citep{estornell2024multi}. More recent work treats the interaction structure as an optimization target, using learnable communication graphs \citep{zhuge2024gptswarm, zhang2025gdesigner}, sparse or hierarchical messaging \citep{li2024improving, qian2025scaling, wang2025talk}, pruning and economical communication pipelines \citep{zhang2025agentprune}, or dynamic team formation \citep{liu2024a} to improve scalability and efficiency.

\subsection{Multi-Model Collaboration, Output Fusion, and Routing}
\label{app:rw_multi_model}

Multi-model collaboration is often organized as an ensemble pipeline that generates candidates, selects among them, and fuses outputs \citep{lu2024merge, chen2025harnessing}. Methods operate at the output level via pairwise ranking \citep{jiang2023llmblender}, consistency-based selection \citep{chen2024universal}, or consensus aggregation \citep{chen2024reconcile}; during reasoning through iterative refinement \citep{wang2025mixture} or tree-search-based decision making \citep{park2025ensembling}; or at decoding time by combining token-level distributions \citep{huang2024ensemble, o2023contrastive, shen2024learning, wang2025speculate}. For deployment, routing and cascading dynamically choose models per input to trade quality for cost \citep{chen2024frugalgpt, shnitzer2024large, ong2025routellm}, with systematic benchmarks characterizing the quality--cost trade-off \citep{hu2024routerbench, ding2025bestroute} and unified analyses of when such strategies are effective \citep{dekoninck2025a}.

\section{Theoretical Analysis and Proofs}
\label{app:theory}

The analysis follows the two questions in \Cref{sec:theoretical_insights}: how the two
heterogeneity signals characterize collective errors, and when their benefits compensate
for a reduction in member quality. Classical ensemble analyses relate diversity to margins
and loss decompositions \citep{tang2006analysis,bian2022diversity,wood2023unified}.
We develop the connection through a choice-probability model, prove an accuracy comparison
with Quality-Only, and relate its terms to the standardized selection objective.
Further results cover soft recovery, profiling perturbations, cross-task pooling, and the
member-specific information available to Stacking. JSD uses base-$2$ logarithms;
cross-entropy and the auxiliary logarithmic calculations below use natural logarithms.

\subsection{Correctness-Level Dependence}
\label{app:correctness_dependence}

Correctness correlations summarize which examples models solve together. We first record their
connection to shared errors, using the classical diversity statistics studied by
\citet{yule1900association,kuncheva2003measures,tang2006analysis}.
Fix a task distribution $\mathcal D$; an empirical development distribution is also allowed.
Write $q_i=\E[C_i]$, $E_i=1-C_i$, and
$\mathrm{DF}_{ij}=\Pr(E_i=E_j=1)$.

\begin{lemma}[Fixed-marginal dependence]
\label{prop:yule_q_monotonicity}
For fixed $q_i,q_j\in(0,1)$, Yule's $Q$, correctness covariance, and double-fault probability
are strictly increasing functions of the joint-correct probability $t=\Pr(C_i=C_j=1)$ over
the interior of its feasible interval. In particular,
\begin{equation}
    \mathrm{DF}_{ij}=(1-q_i)(1-q_j)+\operatorname{Cov}(C_i,C_j).
    \label{eq:double_fault_identity}
\end{equation}
\end{lemma}

\begin{proof}
The four joint probabilities are
$p_{11}=t$, $p_{10}=q_i-t$, $p_{01}=q_j-t$, and
$p_{00}=1-q_i-q_j+t$, where
$\max\{0,q_i+q_j-1\}\le t\le\min\{q_i,q_j\}$.
Thus $\operatorname{Cov}(C_i,C_j)=t-q_iq_j$ and
$\mathrm{DF}_{ij}=1-q_i-q_j+t$, giving \Cref{eq:double_fault_identity}; both have derivative one.
In the interior, all four cells are positive. For the odds ratio
$O(t)=p_{11}p_{00}/(p_{10}p_{01})$,
\begin{equation}
    \frac{d}{dt}\log O(t)
    =\frac1t+\frac1{1-q_i-q_j+t}+\frac1{q_i-t}+\frac1{q_j-t}>0.
\end{equation}
Since $Q=(O-1)/(O+1)$ and $dQ/dO=2/(O+1)^2>0$, $Q$ is strictly increasing as well.
Endpoint values follow by continuity. Finally,
$\operatorname{Cov}(E_i,E_j)=\operatorname{Cov}(C_i,C_j)$ because $E_i=1-C_i$.
\end{proof}

With individual qualities held fixed, $\HIerr$ therefore favors pairs with fewer shared errors.
For two members, an oracle that succeeds whenever either member is correct has error probability
$\mathrm{DF}_{ij}$. Distributional analysis below examines how soft aggregation can recover
information even on examples where every member's top prediction is wrong.

\paragraph{Classical majority-vote bound.}
For completeness, the standard one-sided variance argument gives another interpretation of
correctness dependence. For binary-label prediction, let $k\ge3$ be odd and let correctness indicators have common mean
$p\in(1/2,1)$ and common pairwise correlation $\rho$, and put
$\bar C=k^{-1}\sum_i C_i$. Expanding the variance and applying Cantelli's inequality gives
\begin{equation}
\begin{aligned}
    V(\rho):=\operatorname{Var}(\bar C)
        &=\frac{p(1-p)}{k}\bigl(1+(k-1)\rho\bigr),\\
    \Pr(\text{majority vote errs})
        =\Pr(\bar C\le1/2)
        &\le\frac{V(\rho)}{V(\rho)+(p-1/2)^2}.
\end{aligned}
\label{eq:classical_vote_bound}
\end{equation}
Indeed, the variance contains $k$ diagonal terms $p(1-p)$ and $k(k-1)$ off-diagonal
terms $\rho p(1-p)$, all divided by $k^2$. Cantelli's inequality applied to
$\bar C-p\le-(p-1/2)$ yields the second line.
Writing $A=(p-1/2)^2>0$, the derivative of this upper bound is
$A p(1-p)(k-1)/[k(V(\rho)+A)^2]>0$ throughout the feasible correlation range.
This is a classical bound on binary majority voting; the subsequent results analyze the
choice distributions and selection rule used by \cmas{} directly.

\subsection{Joint Characterization of Collective Failure}
\label{app:joint_failure}

We work on one fixed distribution $\mathcal D$ of labeled examples. The same identities
hold for the uniform distribution on a finite set. Task indices are suppressed within
this analysis; the target-task and pooled profiles are distinguished explicitly in
\Cref{app:selection_gain}.

\paragraph{Model and pairwise quantities.}
All candidate members follow \eqref{eq:joint_model}, with common constants $a,b,c,u$.
For each example, let $w_i\ne y$ denote the preferred answer of an incorrect member.
Writing the true answer first, the distributions are
\[
    p^{\rm c}=(a,u,u,u),\qquad
    p^{w_1}=(a,b,c,c),\quad
    p^{w_2}=(a,c,b,c),\quad
    p^{w_3}=(a,c,c,b).
\]
The correctness indicators and the $w_i$ have an arbitrary joint distribution across
members and examples. Since $c<u<b$ and $a>(2b+u)/3$, we have $a>u$,
so these templates have the stated unique preferred answers. The parameter region
has nonempty interior; for example, $(a,b,c)=(0.42,0.48,0.05)$ satisfies all inequalities.

For a three-member team, define
\begin{align}
    \bar q_S&=\frac13\sum_{i\in S}q_i,\qquad e_S=1-\bar q_S,\notag\\
    D_S&=\frac13\sum_{\{i,j\}\subset S}\Pr(C_i=C_j=0),\notag\\
    K_S&=\frac13\sum_{\{i,j\}\subset S}
        \Pr(C_i=C_j=0,\ w_i=w_j),\notag\\
    J_S&=\frac13\sum_{\{i,j\}\subset S}
        \E[\JSD_2(p_i,p_j)].
    \label{eq:joint_pair_quantities}
\end{align}
The event defining $K_S$ evaluates $w_i,w_j$ only when both members are incorrect.
Here $J_S$ is the population counterpart of the task-specific $\HIdist$.
The pairwise terms in $D_S$ are determined by the marginal qualities and Yule's $Q$
on nondegenerate correctness tables, by \Cref{prop:yule_q_monotonicity}.

\paragraph{Divergence constants.}
Let
\[
    g(x,z)=\frac12\left[x\log_2\frac{2x}{x+z}
                           +z\log_2\frac{2z}{x+z}\right].
\]
The two relevant divergences are
\begin{equation}
    d=2g(b,c)>0,\qquad h=g(b,u)+2g(c,u)>0,\qquad \theta=\frac{2h}{d}.
    \label{eq:joint_divergence_constants}
\end{equation}
The true-answer coordinate contributes zero to both divergences. Two correct
members, or two incorrect members preferring the same answer, have identical
distributions. A correct--incorrect pair contributes $h$, and two incorrect
members preferring different answers contribute $d$.

\begin{lemma}[Positive coefficients in the joint decomposition]
\label{lem:joint_theta}
For $b>c>0$ and $u=(b+2c)/3$, the constants in
\eqref{eq:joint_divergence_constants} satisfy $0<2h<d$.
\end{lemma}

\begin{proof}
Positivity follows from the distinct templates. To prove $d-2h>0$, set
$t=b/c>1$ and $\varphi(z)=z\ln z$. Expanding the scalar divergences and
cancelling the terms containing $\ln c$ gives
\[
    d-2h=\frac{c}{\ln2}F(t),
\]
where
\[
    F(t)=2\varphi\!\left(\frac{2t+1}{3}\right)
        +4\varphi\!\left(\frac{t+5}{6}\right)
        -3\varphi\!\left(\frac{t+2}{3}\right)
        -2\varphi\!\left(\frac{t+1}{2}\right).
\]
Direct differentiation yields $F(1)=F'(1)=0$ and
\begin{equation}
    F''(t)=
    \frac{-2t^2+7t+13}{(t+2)(t+1)(2t+1)(t+5)}.
    \label{eq:joint_theta_derivative}
\end{equation}
This derivative is positive for $1<t<t_*=(7+\sqrt{153})/4$ and negative
for $t>t_*$. Thus $F'$ first increases from zero and then decreases to
\[
    \lim_{t\to\infty}F'(t)=\frac53\ln2-\ln3>0,
\]
where the final inequality is equivalent to $32>27$. Consequently $F'(t)>0$
for every $t>1$, and integration from $1$ gives $F(t)>0$. Dividing
$0<2h<d$ by $d$ proves $0<\theta<1$.
\end{proof}

\begin{proof}[Proof of Proposition~\ref{prop:joint_failure}]
For a pair $(i,j)$, write
$D_{ij}=\Pr(C_i=C_j=0)$ and
$K_{ij}=\Pr(C_i=C_j=0,w_i=w_j)$. The probability that exactly one
member is incorrect is $(1-q_i)+(1-q_j)-2D_{ij}$. The pairwise
divergence calculation above therefore gives
\begin{equation}
    \E[\JSD_2(p_i,p_j)]
    =h\bigl[(1-q_i)+(1-q_j)-2D_{ij}\bigr]
       +d(D_{ij}-K_{ij}).
    \label{eq:joint_pair_identity}
\end{equation}
Each member occurs in two of the three unordered pairs. Averaging yields
$J_S=2h(e_S-D_S)+d(D_S-K_S)$. Rearranging and using
\Cref{lem:joint_theta} establishes \eqref{eq:joint_collision}.

It remains to characterize aggregation. Every member assigns $a$ to the
true answer. If all three are incorrect and prefer the same wrong answer,
that answer has score $b>a$ under Choice-Soft and $b^3>a^3$ under PoE.
Both aggregators therefore err.

In every other case, fix an incorrect answer $r$. If at least one member is
correct, the three probabilities assigned to $r$ have sum at most $2b+u$
and product at most $b^2u$. If all members are incorrect but their preferred
answers are not all equal, at most two assign $b$ to $r$, so the corresponding
bounds are $2b+c<2b+u$ and $b^2c<b^2u$. Thus the true answer has a strictly
larger Choice-Soft score by \eqref{eq:joint_model}. For PoE, the same condition and
the arithmetic--geometric mean inequality imply
\[
    a^3>\left(\frac{2b+u}{3}\right)^3\ge b^2u.
\]
Hence both aggregators err exactly on the event that all three members prefer
the same wrong answer. On that event, every pair contributes one to the
same-wrong-answer indicator. Pointwise,
\[
    \ind\{\hat y_S\ne y\}
    \le\frac13\sum_{\{i,j\}\subset S}
          \ind\{C_i=C_j=0,\ w_i=w_j\}.
\]
Taking expectations gives $\Pr(\hat y_S\ne Y)\le K_S$.
\end{proof}

\subsection{Accuracy Gains over Quality-Only Selection}
\label{app:quality_gain}

We derive a lower bound on the baseline's error probability to accompany the
upper bound in \Cref{prop:joint_failure}. Their combination gives an accuracy
comparison between teams.

\begin{lemma}[Error bounds from pairwise profiles]
\label{lem:joint_risk_bounds}
Under \eqref{eq:joint_model}, let $\mathcal R_S$ be the error probability of
either Choice-Soft or PoE. Then
\begin{equation}
    \max\left\{0,\frac{D_S+3K_S-2e_S}{2}\right\}
    \le\mathcal R_S
    \le\min\{K_S,\ 1-3e_S+3D_S\}.
    \label{eq:joint_sharp_risk}
\end{equation}
Both endpoints are attainable for every feasible triple $(e_S,D_S,K_S)$.
In particular,
\begin{equation}
    0\le K_S-\mathcal R_S
    \le\varepsilon_S:=
       \min\left\{K_S,\ e_S-\frac{D_S+K_S}{2}\right\}.
    \label{eq:joint_baseline_correction}
\end{equation}
\end{lemma}

\begin{proof}
Suppress the team subscript. There are seven possible error patterns, summarized
below. The last four columns give the contribution of each pattern to the
mean member error, double-fault probability, same-wrong-answer probability,
and aggregation error.
\begin{center}
\small
\begin{tabular}{llcccc}
\toprule
Probability & Error pattern & $e$ & $D$ & $K$ & $\mathcal R$\\
\midrule
$p_0$ & No member incorrect & $0$ & $0$ & $0$ & $0$\\
$p_1$ & One member incorrect & $1/3$ & $0$ & $0$ & $0$\\
$p_{2s}$ & Two incorrect, same answer & $2/3$ & $1/3$ & $1/3$ & $0$\\
$p_{2d}$ & Two incorrect, different answers & $2/3$ & $1/3$ & $0$ & $0$\\
$p_{3s}$ & Three incorrect, same answer & $1$ & $1$ & $1$ & $1$\\
$p_{21}$ & Three incorrect, two answers & $1$ & $1$ & $1/3$ & $0$\\
$p_{111}$ & Three incorrect, three answers & $1$ & $1$ & $0$ & $0$\\
\bottomrule
\end{tabular}
\end{center}
These probabilities are nonnegative and sum to one. The table gives
\begin{align}
    e&=\frac{p_1+2p_{2s}+2p_{2d}}3+p_{3s}+p_{21}+p_{111},\notag\\
    D&=\frac{p_{2s}+p_{2d}}3+p_{3s}+p_{21}+p_{111},\notag\\
    K&=\frac{p_{2s}}3+p_{3s}+\frac{p_{21}}3,\qquad
    \mathcal R=p_{3s}.
    \label{eq:joint_pattern_moments}
\end{align}
In particular, $K-\mathcal R=(p_{2s}+p_{21})/3\ge0$. A second identity is
\[
    e-\frac{D+K}{2}-(K-\mathcal R)
       =\frac{p_1}{3}+\frac{p_{2d}+p_{111}}2\ge0.
\]
Since $\mathcal R\ge0$, these two identities imply
\eqref{eq:joint_baseline_correction} and the lower bound in
\eqref{eq:joint_sharp_risk}. Finally,
$1-3e+3D=p_0+p_{3s}+p_{21}+p_{111}\ge\mathcal R$ gives the second upper bound.

To show attainability, fix any feasible $(e,D,K)$ and choose any $f$ between
the two endpoints in \eqref{eq:joint_sharp_risk}. Set
\[
    \tau=\max\{f,\,2D-e,\,0\},\qquad
    z=\max\{0,\,3(K-D+\tau-f)\}.
\]
Feasibility gives $0\le K\le D\le e$ and $D\ge2e-1$.
The bounds on $f$ then imply
\[
    f\le\tau\le
    \min\{D,\ 1-3e+3D,\ f+\tfrac32(D-K)\}.
\]
For the final inequality, the only additional comparison is
$2D-e\le f+\tfrac32(D-K)$, which is exactly the stated lower bound on $f$.
It follows that $0\le z\le\min\{\tau-f,\,3(K-f)\}$.
Define
\begin{align*}
    p_{3s}&=f,&
    p_{21}&=z,&
    p_{111}&=\tau-f-z,\\
    p_{2s}&=3(K-f)-z,&
    p_{2d}&=3(D-\tau)-p_{2s},\\
    p_1&=3e-6D+3\tau,&
    p_0&=1-3e+3D-\tau.
\end{align*}
Every term is nonnegative; for $p_{2d}$ this follows from
$z\ge3(K-D+\tau-f)$. They sum to one, and substitution into
\eqref{eq:joint_pattern_moments} recovers $(e,D,K,\mathcal R)=(e,D,K,f)$.
Each pattern is realizable using the three incorrect choices. Thus the endpoints
are sharp given these aggregate pairwise quantities.
\end{proof}

\begin{proof}[Proof of Proposition~\ref{prop:quality_only_gain}]
Write $\mathcal R_S=1-\operatorname{Acc}(S)$, with the same aggregator for both
teams. By \Cref{lem:joint_risk_bounds},
$\mathcal R_Q\ge K_Q-\varepsilon_Q$ and $\mathcal R_S\le K_S$. Therefore
\[
    \operatorname{Acc}(S)-\operatorname{Acc}(S_Q)
       =\mathcal R_Q-\mathcal R_S
       \ge K_Q-K_S-\varepsilon_Q.
\]
Applying \eqref{eq:joint_collision} to each team and using
$e_Q-e_S=-(\bar q_Q-\bar q_S)$ gives \eqref{eq:quality_only_gain}.
Its positive right-hand side implies a strict accuracy improvement.
If the baseline members make identical predictions, then
$D_Q=K_Q=e_Q$, hence $\varepsilon_Q=0$.
\end{proof}

The correction quantifies how much pairwise error concentration can overstate
the baseline's aggregate error. Its exact excess is
$(p_{2s}+p_{21})/3$: these are cases with a same-wrong-answer pair that the
third member or soft aggregation resolves. The bound on this excess uses
only $(e_Q,D_Q,K_Q)$, with $K_Q$ supplied by the joint decomposition.
The full interval in \eqref{eq:joint_sharp_risk} also gives the sharper comparison
\[
    \operatorname{Acc}(S)-\operatorname{Acc}(S_Q)
    \ge
    \max\left\{0,\frac{D_Q+3K_Q-2e_Q}{2}\right\}
       -\min\{K_S,\ 1-3e_S+3D_S\}.
\]
Equation~\eqref{eq:quality_only_gain} separates this comparison into the two
complementarity gains and the member-quality difference.

\paragraph{Perturbations of the probability templates.}
The joint mechanism extends to a neighborhood of the common templates.
Suppose each actual distribution is within total variation $\delta$ of its
template on every example, where
\begin{equation}
    0\le\delta<c,\qquad
    2\delta<\min\left\{b-a,\ a-\frac{2b+u}{3}\right\}.
    \label{eq:joint_perturbation}
\end{equation}
A coordinate changes by at most $\delta$. An incorrect member retains its
preferred answer because $b-a>2\delta$; a correct member retains its preferred
answer because $a-u>a-(2b+u)/3>2\delta$.
For every pattern with successful aggregation, the correct Choice-Soft score is
at least $a-\delta$ and every wrong score is at most $(2b+u)/3+\delta$.
The same inequality gives
$(a-\delta)^3>(b+\delta)^2(u+\delta)$ by the arithmetic--geometric mean
inequality, preserving PoE recovery. The common wrong answer still wins in
the remaining pattern. Thus $e,D,K$ and both aggregation accuracies are unchanged.

By the JSD gradient argument in \eqref{eq:soft_js_perturbation_bound},
the observed divergence $J'_S$ satisfies
\[
    |J'_S-J_S|\le\eta_\delta,\qquad
    \eta_\delta=2\delta\log_2\frac1{c-\delta}.
\]
Consequently, defining
$K'_S=\theta e_S+(1-\theta)D_S-J'_S/d$ gives
$|K'_S-K_S|\le\eta_\delta/d$ and
$\mathcal R_S\le K'_S+\eta_\delta/d$.
The map $K\mapsto\min\{K,e-(D+K)/2\}$ is $1$-Lipschitz.
After clipping its value below at zero, the baseline correction computed with
$K'_Q$ differs from $\varepsilon_Q$ by at most $\eta_\delta/d$.
Replacing $J_S,J_Q,\varepsilon_Q$ in \eqref{eq:quality_only_gain} by these
perturbed quantities therefore preserves the lower bound after subtracting
$3\eta_\delta/d$.

\subsection{Soft Recovery from Dispersed Incorrect Choices}
\label{app:soft_recovery}

This section isolates recovery on examples where every member is incorrect,
extending the mechanism to general team sizes, larger label sets, and
perturbed choice distributions. We work on a
fixed distribution of labeled examples; the same calculations apply to the
uniform distribution on a finite profiling set. Throughout this section,
Jensen--Shannon divergence is measured in bits \citep{lin1991divergence}:
\begin{equation}
    \JSD_2(p,q)
    = \frac12\sum_r p_r\log_2\frac{2p_r}{p_r+q_r}
      +\frac12\sum_r q_r\log_2\frac{2q_r}{p_r+q_r}.
    \label{eq:soft_js_bits}
\end{equation}

\paragraph{Incorrect-choice occupancy and aggregation margins.}
Consider $k\ge2$ members and $L\ge4$ choices. On an example with true label $y$,
member $i$ has an incorrect preferred choice $w_i\ne y$ and distribution
\begin{equation}
    p_i(r)=
    \begin{cases}
        a, & r=y,\\
        b, & r=w_i,\\
        c, & r\notin\{y,w_i\},
    \end{cases}
    \qquad
    0<c<a<b,\qquad a+b+(L-2)c=1.
    \label{eq:soft_template}
\end{equation}
Thus every member predicts incorrectly, while the true label is each member's
second choice. For each incorrect label $r$, let
\begin{equation}
    n_r=\sum_{i=1}^k\ind\{w_i=r\},\qquad
    v_r=\frac{n_r}{k},\qquad \sum_{r\ne y}v_r=1.
    \label{eq:soft_occupancy}
\end{equation}
Define the divergence between two members with different incorrect preferred
choices and the mean pairwise divergence on this example by
\begin{align}
    d&=b\log_2\frac{2b}{b+c}+c\log_2\frac{2c}{b+c}>0,
    \label{eq:soft_pair_scale}\\
    j_S(x)&=\frac{1}{\binom{k}{2}}
        \sum_{i<j}\JSD_2(p_i,p_j).
    \label{eq:soft_instance_js}
\end{align}

\begin{lemma}[Occupancy and strict recovery]
\label{lem:soft_occupancy}
Under \eqref{eq:soft_template},
\begin{equation}
    j_S(x)
    =d\,\frac{k^2-\sum_{r\ne y}n_r^2}{k(k-1)}
    =d\,\frac{k}{k-1}\left(1-\sum_{r\ne y}v_r^2\right).
    \label{eq:soft_js_occupancy}
\end{equation}
Choice-Soft assigns the true label a strictly larger score than every incorrect
label if and only if
\begin{equation}
    \max_{r\ne y}v_r<\frac{a-c}{b-c}.
    \label{eq:soft_cs_threshold}
\end{equation}
For PoE, the corresponding necessary and sufficient condition is
\begin{equation}
    \max_{r\ne y}v_r<\frac{\log_2(a/c)}{\log_2(b/c)}.
    \label{eq:soft_poe_threshold}
\end{equation}
Equality in either threshold gives a tie between the true label and a highest
scoring incorrect label for that aggregator.
\end{lemma}

\begin{proof}
When $w_i=w_j$, the two distributions coincide. When $w_i\ne w_j$, they differ
only at these two incorrect labels, with probabilities $(b,c)$ and $(c,b)$.
Substitution into \eqref{eq:soft_js_bits} gives
\begin{equation}
    \JSD_2(p_i,p_j)=d\,\ind\{w_i\ne w_j\}.
    \label{eq:soft_pair_indicator}
\end{equation}
The number of unordered pairs with different incorrect preferred choices is
\[
    \binom{k}{2}-\sum_{r\ne y}\binom{n_r}{2}
    =\frac{k^2-\sum_{r\ne y}n_r^2}{2},
\]
which proves \eqref{eq:soft_js_occupancy}.

Choice-Soft assigns probability $a$ to the true label and
$c+(b-c)v_r$ to incorrect label $r$. Its margin against its strongest
incorrect competitor is therefore
\begin{equation}
    a-c-(b-c)\max_{r\ne y}v_r.
    \label{eq:soft_cs_margin}
\end{equation}
This margin is positive exactly under \eqref{eq:soft_cs_threshold}.
For PoE, the unnormalized scores are $a^k$ for the true label and
$b^{n_r}c^{k-n_r}$ for incorrect label $r$. The base-two log score ratio, divided by $k$,
against the strongest incorrect competitor is
\begin{equation}
    \log_2(a/c)-\max_{r\ne y}v_r\log_2(b/c).
    \label{eq:soft_poe_margin}
\end{equation}
Since $b>c>0$, this expression is positive exactly under
\eqref{eq:soft_poe_threshold}. The common positive PoE normalizer does not affect
the comparison. These expressions also establish the statements about ties.
\end{proof}

The divergence records the squared concentration $\sum_r v_r^2$ of incorrect
preferred choices, whereas recovery depends on their largest concentration
$\max_r v_r$. This distinction explains both the information supplied by
choice-space divergence and the information lost by averaging it. Within this
family, the PoE threshold is larger than the Choice-Soft threshold: strict
concavity of $\log_2$ gives
\[
    \frac{\log_2(a/c)}{\log_2(b/c)}>\frac{a-c}{b-c}
    \qquad(c<a<b).
\]

\paragraph{Three-member recovery and sharp bounds.}
\begin{proposition}[Distributional heterogeneity and recovery]
\label{prop:soft_recovery}
Let $k=3$ and suppose \eqref{eq:soft_template} holds on the event $H$ that
all three members predict incorrectly, with $\Pr(H)>0$ and $a>(2b+c)/3$.
Define $U=\E[j_S(X)\mid H]/d$, with $j_S,d$ from
\eqref{eq:soft_instance_js} and \eqref{eq:soft_pair_scale}.
For either Choice-Soft or PoE, the conditional recovery probability
$G=\Pr(\hat y_S=Y\mid H)$ satisfies the sharp bounds
\begin{equation}
    U\le G\le\min\{1,3U/2\}.
    \label{eq:recovery_main}
\end{equation}
Both aggregators recover exactly when the three preferred wrong answers are
not all identical.
\end{proposition}

\begin{proof}[Proof of Proposition~\ref{prop:soft_recovery}]
Condition on $H$, the event that all three members predict incorrectly. If all
three incorrect preferred choices coincide, that choice receives score $b>a$
under Choice-Soft and unnormalized score $b^3>a^3$ under PoE, so both aggregators
fail. Otherwise, no incorrect label is preferred by more than two members.
The condition $a>(2b+c)/3$ makes the Choice-Soft margin strictly positive even
for an incorrect label preferred by two members. Moreover,
\[
    a^3>\left(\frac{2b+c}{3}\right)^3\ge b^2c,
\]
so PoE also recovers in this case. Thus both aggregators recover on $H$ exactly
when the three incorrect preferred choices are not all equal.

Let $r_3$, $r_{21}$, and $r_{111}$ denote, conditional on $H$, the probabilities
of all three choices being equal, exactly two being equal, and all three being
different, respectively. They are nonnegative and sum to one.
Equation~\eqref{eq:soft_pair_indicator} gives mean pairwise divergences $0$,
$2d/3$, and $d$ in these three cases. Consequently, the quantities $U$ and $G$
in the proposition satisfy
\begin{equation}
    U=\frac23r_{21}+r_{111},\qquad
    G=r_{21}+r_{111}.
    \label{eq:soft_three_patterns}
\end{equation}
In particular,
\[
    G-U=\frac13r_{21}\ge0,\qquad
    \frac32U-G=\frac12r_{111}\ge0,
\]
and $G\le1$, proving $U\le G\le\min\{1,3U/2\}$.

Both endpoints are attainable for every $U\in[0,1]$. The lower endpoint follows
from $(r_3,r_{21},r_{111})=(1-U,0,U)$. For the upper endpoint, take
\begin{equation}
    (r_3,r_{21},r_{111})=
    \begin{cases}
        (1-3U/2,\,3U/2,\,0), & 0\le U\le2/3,\\
        (0,\,3(1-U),\,3U-2), & 2/3\le U\le1.
    \end{cases}
    \label{eq:soft_sharp_constructions}
\end{equation}
There are at least three incorrect labels because $L\ge4$, so all three
patterns can be realized. These constructions use the same $a,b,c$ and can
be combined with any fixed distribution of predictions outside $H$.

Finally, changing the joint pattern of $(w_1,w_2,w_3)$ on $H$ leaves each
correctness indicator equal to zero there. With $H$ and all predictions outside
$H$ fixed, the entire correctness vector is unchanged. Hence every member's
accuracy and every pair's correctness contingency table are unchanged, as is
Yule's $Q$, including the zero-denominator convention in \eqref{eq:yule_q}.
\end{proof}

\paragraph{Relation to overall accuracy and the profiled divergence.}
Let $\tau=\Pr(H)>0$, and write
\begin{equation}
    J_{\mathrm{all}}=\E[j_S(X)],\qquad
    J_{\mathrm{out}}=\E[\ind\{H^c\}j_S(X)].
    \label{eq:soft_overall_js_def}
\end{equation}
Here $J_{\mathrm{all}}$ is precisely the population version of
$\HIdist(S;J^t)$ for the fixed task distribution. For aggregator
$\mathcal{A}\in\{\textsc{cs},\textsc{poe}\}$, define
$A_{\mathrm{out}}^{\mathcal{A}}=
\Pr(H^c,\hat y_{\mathcal{A}}(X)=Y)$ and let
$\operatorname{Acc}_{\mathcal{A}}=\Pr(\hat y_{\mathcal{A}}(X)=Y)$.
Partitioning by $H$ gives
\begin{equation}
    J_{\mathrm{all}}=J_{\mathrm{out}}+\tau dU,\qquad
    \operatorname{Acc}_{\mathcal{A}}
        =A_{\mathrm{out}}^{\mathcal{A}}+\tau G.
    \label{eq:soft_overall_decomposition}
\end{equation}
The sharp conditional bounds therefore imply
\begin{equation}
\begin{aligned}
    A_{\mathrm{out}}^{\mathcal{A}}
       +\frac{J_{\mathrm{all}}-J_{\mathrm{out}}}{d}
    &\le \operatorname{Acc}_{\mathcal{A}}\\
    &\le A_{\mathrm{out}}^{\mathcal{A}}
       +\min\!\left\{\tau,
          \frac{3(J_{\mathrm{all}}-J_{\mathrm{out}})}{2d}\right\}.
\end{aligned}
    \label{eq:soft_overall_bounds}
\end{equation}
These bounds remain sharp when the predictions outside $H$ are fixed.

For example, suppose that with probability $q\in(0,1)$ all three members output
the same positive distribution whose unique preferred choice is correct, and
the remaining probability $1-q$ follows the model on $H$. Then $q_i=q$ and
$Q_{ij}=1$ for every pair, while $J_{\mathrm{out}}=0$ and
$A_{\mathrm{out}}^{\mathcal{A}}=q$. Equation~\eqref{eq:soft_overall_bounds}
reduces to
\begin{equation}
    q+\frac{J_{\mathrm{all}}}{d}
    \le \operatorname{Acc}_{\mathcal{A}}
    \le q+\min\!\left\{1-q,\frac{3J_{\mathrm{all}}}{2d}\right\}.
    \label{eq:soft_identical_background}
\end{equation}
Thus the existing choice-space divergence can quantify recovery opportunities
while the full correctness profile is held fixed. The calculation applies
task by task; the pooled $\HIdist$ is the arithmetic mean of the
corresponding taskwise divergences.

\paragraph{Stability under perturbations of member distributions.}
The recovery mechanism persists when the members' probabilities differ from
the common template. We use total variation in the convention
$\operatorname{TV}(p,q)=\tfrac12\|p-q\|_1$.

\begin{corollary}[Perturbed soft recovery]
\label{cor:soft_recovery_robust}
Let $k=3$, and let $p_i^0$ have the form \eqref{eq:soft_template} on an event
$H$ of positive probability. Suppose the actual distributions $p_i$ satisfy,
almost surely on $H$,
\begin{equation}
    \operatorname{TV}(p_i,p_i^0)\le\varepsilon
    \quad(i=1,2,3),\qquad
    0\le\varepsilon<c,\qquad b-a>2\varepsilon,
    \label{eq:soft_perturbation_assumptions}
\end{equation}
and
\begin{equation}
    a-\frac{2b+c}{3}>2\varepsilon.
    \label{eq:soft_perturbation_cs}
\end{equation}
Each member still has unique preferred choice $w_i\ne y$, and both Choice-Soft
and PoE recover on $H$ exactly when $(w_1,w_2,w_3)$ are not all equal. Let $G$
be this conditional recovery probability and set
\begin{equation}
    J_H=\E\!\left[
        \frac13\sum_{i<j}\JSD_2(p_i,p_j)\,\middle|\,H\right],
    \qquad
    \eta=2\varepsilon\log_2\frac{1}{c-\varepsilon}.
    \label{eq:soft_perturbed_js}
\end{equation}
With $d$ as in \eqref{eq:soft_pair_scale},
\begin{equation}
    \max\!\left\{0,\frac{J_H-\eta}{d}\right\}
    \le G\le
    \min\!\left\{1,\frac{3(J_H+\eta)}{2d}\right\}.
    \label{eq:soft_perturbed_bounds}
\end{equation}
For PoE alone, the conclusions hold with \eqref{eq:soft_perturbation_cs}
replaced by
\begin{equation}
    (a-\varepsilon)^3>(b+\varepsilon)^2(c+\varepsilon).
    \label{eq:soft_perturbation_poe}
\end{equation}
\end{corollary}

\begin{proof}
For probability distributions, a total-variation bound of $\varepsilon$ implies
an absolute difference of at most $\varepsilon$ in every coordinate. Therefore
\[
    p_i(w_i)\ge b-\varepsilon>a+\varepsilon\ge p_i(y),
\]
and $p_i(w_i)>p_i(r)$ for every other incorrect label $r$, since $a>c$.
Thus each $w_i$ remains the unique preferred choice. If all three $w_i$ agree,
their common label strictly exceeds the true label for every member, so both
the arithmetic mean and the product favor an incorrect label.

If the $w_i$ are not all equal, any incorrect label is preferred by at most two
members. Its Choice-Soft score is at most $(2b+c)/3+\varepsilon$, whereas the
true label's score is at least $a-\varepsilon$.
Condition~\eqref{eq:soft_perturbation_cs} makes the latter strictly larger.
For PoE, any incorrect label has unnormalized score at most
$(b+\varepsilon)^2(c+\varepsilon)$, while the true label has score at least
$(a-\varepsilon)^3$. All factors are positive because $\varepsilon<c<a$.
Condition~\eqref{eq:soft_perturbation_poe} therefore suffices. It is also
implied by \eqref{eq:soft_perturbation_cs}, since
\[
    a-\varepsilon>
    \frac{2(b+\varepsilon)+(c+\varepsilon)}3
    \ge\bigl((b+\varepsilon)^2(c+\varepsilon)\bigr)^{1/3}.
\]
This proves the stated recovery rule for both aggregators.

To bound the change in divergence, write $F(p,q)=\JSD_2(p,q)$ and
$m_\varepsilon=c-\varepsilon>0$. Along the straight line from any template
pair $(p^0,q^0)$ to its perturbed pair $(p,q)$, every coordinate of both
distributions belongs to $[m_\varepsilon,1]$. Direct differentiation gives
\begin{equation}
    \frac{\partial F}{\partial p_r}
        =\frac12\log_2\frac{2p_r}{p_r+q_r},
    \qquad
    \frac{\partial F}{\partial q_r}
        =\frac12\log_2\frac{2q_r}{p_r+q_r}.
    \label{eq:soft_js_gradient}
\end{equation}
On this line,
\[
    m_\varepsilon
    \le\frac{2p_r}{p_r+q_r}\le\frac1{m_\varepsilon},
\]
and the same inequalities hold after interchanging $p_r$ and $q_r$.
Consequently, both gradient vectors have infinity norm at most
$\tfrac12\log_2(1/m_\varepsilon)$. Integrating the directional derivative
along the line yields
\begin{align}
    |F(p,q)-F(p^0,q^0)|
    &\le\frac12\log_2\frac1{m_\varepsilon}
       \bigl(\|p-p^0\|_1+\|q-q^0\|_1\bigr)\notag\\
    &\le 2\varepsilon\log_2\frac1{c-\varepsilon}=\eta.
    \label{eq:soft_js_perturbation_bound}
\end{align}
Averaging over the three pairs and then conditioning on $H$ preserves this
bound. If $U$ denotes the normalized conditional divergence of the template
distributions, we obtain $|J_H-dU|\le\eta$.
The recovered cases are exactly the same occupancy patterns as for the
templates, so \eqref{eq:soft_three_patterns} gives
$U\le G\le\min\{1,3U/2\}$. Combining these inequalities with
$(J_H-\eta)/d\le U\le(J_H+\eta)/d$ proves
\eqref{eq:soft_perturbed_bounds}.
\end{proof}

Every template satisfying the strict condition in
Proposition~\ref{prop:soft_recovery} admits a positive perturbation radius:
any
\[
    0<\varepsilon<
    \min\!\left\{c,\frac{b-a}{2},
        \frac12\left(a-\frac{2b+c}{3}\right)\right\}
\]
satisfies the corollary. Hence recovery is stable on a neighborhood of member
distributions, with an explicit error allowance for the observed JSD.

\subsection{Stability of Multi-Start Selection}
\label{app:selection_stability}

We analyze two finite profiling outputs for the same target task, candidate set, team size,
and heterogeneity weights. Each output is standardized separately. Let $P=\binom{m}{2}$,
$L_r=\binom{r}{2}$, and $E(S)=\{\{i,j\}:i,j\in S,\ i<j\}$. Edge vectors contain the $P$
upper-triangular entries, each counted once; all edge norms below use this convention.
For a vector $x\in\mathbb{R}^d$, define population standardization by
\begin{equation}
    C_d=I_d-\frac{\mathbf{1}\mathbf{1}^{\top}}{d},\qquad
    \sigma(x)=\frac{\|C_dx\|_2}{\sqrt d},\qquad
    Z(x)=
    \begin{cases}
        C_dx/\sigma(x), & \sigma(x)>0,\\
        0, & \sigma(x)=0.
    \end{cases}
    \label{eq:selection_base_standardization}
\end{equation}
Thus $\|Z(x)\|_2=\sqrt d$ for every nonconstant vector. The corresponding quality and
combined edge signals are
\begin{equation}
    u_i=Z(q^t)_i,\qquad
    v_{ij}=\lambda_1Z(1-\bar R)_{ij}+\lambda_2Z(\bar J)_{ij},
    \label{eq:selection_standardized_signals}
\end{equation}
where matrix standardization acts on the upper-triangular vector. The quality statistics are
computed over the $m$ candidates, and each HI channel is standardized over its $P$ base pairs.
For $|S|=r\ge2$, the implemented objective in its active normalization regime is exactly
\begin{equation}
    F(S)=\frac{1}{\sqrt r}\sum_{i\in S}u_i
    +\frac{1}{\sqrt{L_r}}\sum_{e\in E(S)}v_e,
    \qquad F(\{i\})=u_i.
    \label{eq:selection_sum_form}
\end{equation}
In particular, $u$ always uses the target task's quality estimates. Write the second profiling
output as $u',v'$, its score as $F'$, and the changes as
$\delta u=u'-u$, $\delta v=v'-v$, and $\delta F(S)=F'(S)-F(S)$.

\paragraph{Normalization in the implementation.}
The default normalizer uses population standard deviation (\texttt{std\_ddof=0}). It sets a
term to zero when $\sigma/\sqrt{n_{\rm items}}\le10^{-12}$. Hence
\eqref{eq:selection_sum_form} agrees with the numerical implementation when constant
channels are zero and every nonconstant channel remains active at each relevant team size.
Sufficient conditions are $\sigma(q^t)>10^{-12}\sqrt{k}$ for a nonconstant quality channel
and $\sigma(h)>10^{-12}\sqrt{L_k}$ for either nonconstant pooled HI channel, in each
profiling output. A channel suppressed at every relevant size can instead be represented
by zero in the effective profile. If a channel changes activation with team size, the same
comparison argument below applies to the effective size-dependent vectors
$u^{(r)},v^{(r)}$, using their perturbations at size $r=s+1$ for an extension and at size
$r=k$ for the final comparison.

\begin{proposition}[Stability of multi-start selection]
\label{prop:selection_stability}
Fix the candidates, $2\le k\le m$, weights, and tie-breaking, with each
nonconstant standardization denominator above the numerical cutoff at every
team size in both profiles. At a reference state of size $s<k$, the change
in the score difference between any two candidate extensions is at most
\begin{equation}
    B_s=\frac{\sqrt{2}\|\delta u\|_2+2\|\delta v\|_2}{\sqrt{s+1}}.
    \label{eq:stability_main}
\end{equation}
If every winning extension in every seed's reference path leads its alternatives
by more than $B_s$, and the final winner leads every other distinct terminal
team by more than $\sqrt2(\|\delta u\|_2+\|\delta v\|_2)$, all paths and the
returned team remain unchanged.
\end{proposition}

\begin{proof}[Proof of Proposition~\ref{prop:selection_stability}]
\emph{Candidate comparisons.}
At a reference state $S$ with $|S|=s<k$, consider two remaining candidates $j$ and $\ell$.
Define
\[
    G_S(j,\ell)=F(S\cup\{j\})-F(S\cup\{\ell\}).
\]
Both extended teams have size $s+1$. Expanding \eqref{eq:selection_sum_form} cancels the
quality of the existing members and all edges within $S$, giving
\[
    G_S(j,\ell)=\frac{u_j-u_\ell}{\sqrt{s+1}}
    +\frac{\sum_{i\in S}(v_{ij}-v_{i\ell})}{\sqrt{L_{s+1}}}.
\]
Comparing marginal gains gives the same expression, since the current score $F(S)$ also
cancels. Subtracting the corresponding expression for the two profiles yields the exact
perturbation
\begin{equation}
    \delta G_S(j,\ell)
    =\frac{\delta u_j-\delta u_\ell}{\sqrt{s+1}}
    +\frac{\sum_{i\in S}(\delta v_{ij}-\delta v_{i\ell})}{\sqrt{L_{s+1}}}.
    \label{eq:selection_extension_perturbation}
\end{equation}
The quality coefficient vector has two nonzero entries, each of magnitude
$1/\sqrt{s+1}$, and therefore norm $\sqrt{2/(s+1)}$. The edge coefficient vector has
$2s$ nonzero entries: the $s$ edges connecting $j$ to $S$ and the $s$ edges connecting
$\ell$ to $S$. These edges are distinct because $j,\ell\notin S$. Its norm is
\[
    \sqrt{\frac{2s}{L_{s+1}}}=\frac{2}{\sqrt{s+1}}.
\]
Cauchy--Schwarz and the triangle inequality therefore give
\begin{equation}
    |\delta G_S(j,\ell)|
    \le \frac{\sqrt2\|\delta u\|_2+2\|\delta v\|_2}{\sqrt{s+1}}
    = B_s.
    \label{eq:selection_extension_bound}
\end{equation}

\emph{Complete-team comparisons.}
For two size-$k$ teams $A,B$, subtraction similarly gives
\begin{equation}
    \begin{aligned}
        \delta\bigl(F(A)-F(B)\bigr)
        ={}&\frac{\sum_{i\in A\setminus B}\delta u_i
            -\sum_{i\in B\setminus A}\delta u_i}{\sqrt k}\\
        &+\frac{\sum_{e\in E(A)\setminus E(B)}\delta v_e
            -\sum_{e\in E(B)\setminus E(A)}\delta v_e}{\sqrt{L_k}}.
    \end{aligned}
    \label{eq:selection_complete_perturbation}
\end{equation}
Let $r=|A\cap B|$. The quality difference has $2(k-r)$ nonzero coefficients.
Moreover, $E(A)\cap E(B)=E(A\cap B)$, so each team has $L_k-L_r$ edges absent from the
other. The same norm calculation gives the overlap-sensitive bound
\begin{equation}
    \begin{aligned}
        \left|\delta\bigl(F(A)-F(B)\bigr)\right|
        &\le \sqrt{2-\frac{2r}{k}}\,\|\delta u\|_2
        +\sqrt{2-\frac{2L_r}{L_k}}\,\|\delta v\|_2\\
        &\le \sqrt2\bigl(\|\delta u\|_2+\|\delta v\|_2\bigr)
        =B_{\rm fin}.
    \end{aligned}
    \label{eq:selection_overlap_bound}
\end{equation}

\emph{Preservation of every seed run.}
Fix one seed. Both runs start with the same singleton. Suppose inductively that they have
reached the same reference state $S$, and let $j_S$ be the reference winner at this state.
For every competing candidate $\ell$, the assumed reference margin and
\eqref{eq:selection_extension_bound} imply
\[
    G'_S(j_S,\ell)
    =G_S(j_S,\ell)+\delta G_S(j_S,\ell)
    \ge G_S(j_S,\ell)-B_s>0.
\]
The perturbed run therefore chooses $j_S$ as well. Induction to size $k$ proves that this
seed reaches the same terminal team. Applying the argument to every seed preserves all
terminal teams.

Let $\mathcal C$ be the set of \emph{distinct} terminal teams in the reference run, and let
$S_{\rm g}$ be its returned team. Repeated occurrences of the same team among different
seeds do not create competing teams. For any $T\in\mathcal C\setminus\{S_{\rm g}\}$, the
assumed final margin and \eqref{eq:selection_overlap_bound} give
\[
    F'(S_{\rm g})-F'(T)
    \ge F(S_{\rm g})-F(T)-B_{\rm fin}>0.
\]
Thus the perturbed multi-start run returns $S_{\rm g}$. If there is only one distinct
terminal team, this final condition is vacuous. The induction follows exactly the seed
paths and terminal-team comparisons performed by Algorithm~\ref{alg:greedy}.
\end{proof}

\paragraph{A sharper certificate using signed comparisons.}
The norm bounds can be replaced by the exact projected perturbations. At a reference
state $S$, write
\[
    w_S(j)=\frac{\delta u_j}{\sqrt{s+1}}
        +\frac{\sum_{i\in S}\delta v_{ij}}{\sqrt{L_{s+1}}},
    \qquad \gamma_S(\ell)=G_S(j_S,\ell).
\]
The reference winner remains a strict winner at this state if and only if
\begin{equation}
    \max_{\ell\in\mathcal M\setminus(S\cup\{j_S\})}
    \bigl\{w_S(\ell)-w_S(j_S)-\gamma_S(\ell)\bigr\}<0.
    \label{eq:selection_signed_extension}
\end{equation}
Indeed, the expression inside the braces is exactly
$F'(S\cup\{\ell\})-F'(S\cup\{j_S\})$. The corresponding final condition is
\begin{equation}
    \max_{T\in\mathcal C\setminus\{S_{\rm g}\}}
    \bigl\{\delta F(T)-\delta F(S_{\rm g})
        -[F(S_{\rm g})-F(T)]\bigr\}<0.
    \label{eq:selection_signed_final}
\end{equation}
An empty maximum is $-\infty$. These conditions retain cancellation between quality and
HI perturbations and require only reference-winner comparisons. Applying them at every
reference state, followed by \eqref{eq:selection_signed_final}, gives the same induction.
Repeated reference states can share a single condition. Strict inequalities remove tie
ambiguity; a non-strict comparison is sufficient when a retained tie is resolved in favor
of the reference choice by the fixed priority rule.

Both versions are deterministic statements about finite profiles, including estimates
that share development examples or models. Their conditions preserve the selected team
relative to the reference profile.

\paragraph{Piecewise-constant selection as weights vary.}
Fix the profiling outputs and vary $(\lambda_1,\lambda_2)\in\mathbb R_{\ge0}^2$. Every
candidate comparison and complete-team comparison is affine in these two weights. There
are finitely many possible states and teams, so the nontrivial comparison equalities form
a finite collection of lines. Within each open cell of their arrangement, all comparison
signs are fixed. Induction over the seed runs and the final comparison then shows that the
returned team is constant throughout the cell. Identically tied comparisons are handled
by the fixed priority rule. A team's selection region may be a union of cells and need
not be convex. This accounts for constant-selection regions in a weight sweep; one color
in an accuracy heatmap can also represent several teams with the same accuracy.

\subsection{Pooling Geometry and Selection Perturbations}
\label{app:pooling_geometry}

For each source task $\tau\in\mathcal T$, collect the upper-triangular HI entries as
\[
    h_1^\tau=(1-R_{ij}^\tau)_{i<j},\qquad
    h_2^\tau=(J_{ij}^\tau)_{i<j},\qquad
    \bar h_\ell=\frac1{|\mathcal T|}\sum_{\tau\in\mathcal T}h_\ell^\tau,
    \quad \ell\in\{1,2\}.
\]
Thus $\bar h_1=1-\bar R$ follows from averaging the task-level Yule-$Q$ statistics in
\eqref{eq:all_pooling}. In this notation the selector uses
$v=\lambda_1Z(\bar h_1)+\lambda_2Z(\bar h_2)$ while retaining $u=Z(q^t)$ from the target
task. All identities below use population standard deviation as in
\eqref{eq:selection_base_standardization}. Their use in the numerical selector inherits the
active-channel conditions above; a numerically suppressed channel contributes zero to its
effective profile.

\paragraph{The direction of the standardized pool.}
For one HI channel, abbreviate $h^\tau=h_\ell^\tau$, and set
\[
    \sigma_\tau=\sigma(h^\tau),\qquad
    z_\tau=Z(h^\tau),\qquad
    W=\sum_{\tau\in\mathcal T}\sigma_\tau z_\tau.
\]
Constant task vectors have $\sigma_\tau=0$ and contribute zero to $W$. If $W\ne0$, then
\begin{equation}
    Z(\bar h)=\sqrt P\,\frac{W}{\|W\|_2}.
    \label{eq:pooling_standardized_direction}
\end{equation}
To prove the identity, center the average using $C_P$ from
\eqref{eq:selection_base_standardization}:
\[
    C_P\bar h
    =\frac1{|\mathcal T|}\sum_{\tau\in\mathcal T}C_Ph^\tau
    =\frac{W}{|\mathcal T|},\qquad
    \sigma(\bar h)=\frac{\|W\|_2}{|\mathcal T|\sqrt P}.
\]
Their ratio gives \eqref{eq:pooling_standardized_direction}. If $W=0$, the pooled vector
is constant and its standardized value is zero. Equal weighting of the raw task vectors
therefore combines their standardized directions with weights proportional to their
within-task dispersion across model pairs. Task-wide offsets vanish under centering,
whereas opposing centered directions can cancel in $W$.

\paragraph{Positive-affine invariance.}
Suppose, for each channel $\ell$, the pooled and target-task vectors obey
\begin{equation}
    \bar h_\ell=\alpha_\ell h_\ell^t+\beta_\ell\mathbf1,
    \qquad \alpha_\ell>0.
    \label{eq:pooling_positive_affine}
\end{equation}
Then $Z(\bar h_\ell)=Z(h_\ell^t)$. For a nonconstant vector this follows because
$C_P\bar h_\ell=\alpha_\ell C_Ph_\ell^t$ and
$\sigma(\bar h_\ell)=\alpha_\ell\sigma(h_\ell^t)$; if the target vector is constant,
both standardized vectors are zero. Consequently, selection with pooled HI and selection
with target-task HI have identical scores at every team size when both use the same
$q^t$. Their greedy paths and returned teams are identical, including ties resolved by the
same rule. This identity applies to the numerical implementation when the scaling
preserves the active or zero status of each channel at every relevant team size.

A sufficient source-level condition for \eqref{eq:pooling_positive_affine} is
$h_\ell^\tau=\alpha_{\tau\ell}h_\ell^\star+\beta_{\tau\ell}\mathbf1$ with
$\alpha_{\tau\ell}>0$ for the source tasks and the target. Averaging preserves this
common centered direction. The condition concerns each channel separately and imposes
no sharing requirement on target qualities $q^t$. Different target tasks can therefore
select different teams. Positive-affine agreement preserves standardized spacings as
well as pair order; agreement of pair order alone does not preserve sums of edge values.

\paragraph{Direction changes bound the effective edge perturbation.}
Consider any two nonconstant vectors $h,h'\in\mathbb R^P$, and let
\[
    \rho(h,h')=\frac{Z(h)^\top Z(h')}{P}.
\]
Their standardized norms are both $\sqrt P$, so expansion of the squared distance gives
\begin{equation}
    \|Z(h')-Z(h)\|_2^2=2P\bigl(1-\rho(h,h')\bigr).
    \label{eq:pooling_direction_distance}
\end{equation}
When the pooled vectors are nonconstant in both profiles for each channel, let
$d_\ell=Z(\bar h'_\ell)-Z(\bar h_\ell)$ and
$\rho_\ell=\rho(\bar h_\ell,\bar h'_\ell)$. Then
\begin{equation}
    \delta v=\lambda_1d_1+\lambda_2d_2,\qquad
    \|\delta v\|_2
    \le\sqrt{2P}\sum_{\ell=1}^2\lambda_\ell\sqrt{1-\rho_\ell}.
    \label{eq:pooling_edge_perturbation}
\end{equation}
The direct norm $\|\lambda_1d_1+\lambda_2d_2\|_2$ in the selection proposition retains
any cancellation between channels. If a channel is constant, the exact difference
$d_\ell$ remains defined by the zero convention, without requiring a correlation value.
Changing only the HI pool leaves $\delta u=0$; re-estimating the development profiles can
also change the target quality vector, in which case its $\delta u$ term is retained.

The pool-to-target alignment can also be expressed using task-to-task alignments. For
one channel, let $\rho_{\tau\nu}=z_\tau^\top z_\nu/P$ for nonconstant task vectors.
For a nonconstant target vector and $W\ne0$, substitution into
\eqref{eq:pooling_standardized_direction} gives
\begin{equation}
    \rho(h^t,\bar h)
    =\frac{\sum_\tau\sigma_\tau\rho_{t\tau}}
    {\sqrt{\sum_{\tau,\nu}\sigma_\tau\sigma_\nu\rho_{\tau\nu}}},
    \label{eq:pooling_alignment_decomposition}
\end{equation}
where zero-dispersion source vectors are omitted from the sums. Indeed,
$z_t^\top W=P\sum_\tau\sigma_\tau\rho_{t\tau}$ and
$\|W\|_2^2=P\sum_{\tau,\nu}\sigma_\tau\sigma_\nu\rho_{\tau\nu}$.
Equations~\eqref{eq:pooling_direction_distance}--\eqref{eq:pooling_alignment_decomposition}
connect task agreement to the perturbation that enters the actual selection comparisons.
Entrywise variance reduction under resampling is a different quantity: after
standardization, stability also depends on the pooled direction and the decision margins.

\subsection{From Selection Scores to Accuracy Gains}
\label{app:selection_gain}

We connect the accuracy comparison in \Cref{prop:quality_only_gain} to the
quality and HI terms in the implemented objective. The connection separates
the curvature of the Q-to-error relation, the choice of objective weights,
and the change from a target-task profile to a pooled development profile.
All teams in this subsection have size three and follow \eqref{eq:joint_model}.

\paragraph{From Yule's Q to a positive three-term score.}
Write $e_i=1-q_i$. For a pair with error probabilities $x,z$ and Yule's
coefficient $r$, let $f(x,z,r)$ denote its double-fault probability.
On positive contingency tables it is the unique solution $s$ of
\begin{equation}
    \Phi(s,x,z):=
    \ln\frac{s(1-x-z+s)}{(x-s)(z-s)}
       =\ln\frac{1+r}{1-r}.
    \label{eq:score_df_inverse}
\end{equation}
The feasible interior is
$\max\{0,x+z-1\}<s<\min\{x,z\}$, with $0<x,z<1$ and $-1<r<1$.
The left-hand side increases from $-\infty$ to $+\infty$ on this interval.
Define
\[
    T=\partial_s\Phi
      =\frac1s+\frac1{1-x-z+s}+\frac1{x-s}+\frac1{z-s}>0.
\]
Implicit differentiation gives
\begin{equation}
    \partial_x f=
       \frac{(1-x-z+s)^{-1}+(x-s)^{-1}}{T}\in(0,1),
    \qquad
    \partial_r f=\frac{2}{(1-r^2)T}>0,
    \label{eq:score_df_derivatives}
\end{equation}
and the expression for $\partial_z f$ follows by symmetry.

Choose an interior reference $(e_*,e_*,r_*)$ and put
$f_*=f(e_*,e_*,r_*)$,
$a_e=\partial_x f(e_*,e_*,r_*)=\partial_z f(e_*,e_*,r_*)$, and
$a_r=\partial_r f(e_*,e_*,r_*)$.
For each actual pair define its residual by the exact identity
\begin{equation}
    D_{ij}
    =f_*+a_e[(e_i-e_*)+(e_j-e_*)]
        +a_r(R^t_{ij}-r_*)+\rho_{ij}.
    \label{eq:score_df_expansion}
\end{equation}
This definition also applies to boundary tables using their actual $D_{ij}$.
If the segments from the reference to $(e_i,e_j,R^t_{ij})$ lie in a compact
interior region on which $\|\nabla^2 f\|_{\rm op}\le M$, Taylor's theorem gives
the quantitative bound
\begin{equation}
    |\rho_{ij}|\le\frac M2
        \left[(e_i-e_*)^2+(e_j-e_*)^2+(R^t_{ij}-r_*)^2\right].
    \label{eq:score_df_remainder}
\end{equation}
Such a finite $M$ exists on every compact region of positive contingency
tables, by the implicit function theorem and \eqref{eq:score_df_derivatives}.

Let $H_S=\HIerr(S;R^t)=1-\frac13\sum_{\{i,j\}\subset S}R^t_{ij}$.
Averaging \eqref{eq:score_df_expansion} over a team, then substituting into
\eqref{eq:joint_collision}, yields
\begin{equation}
    -K_S=c_0+A_q\bar q_S+A_H H_S+A_J J_S+\eta_S,
    \label{eq:score_collision_expansion}
\end{equation}
where $c_0$ is independent of the team and
\begin{equation}
    A_q=\theta+2(1-\theta)a_e>0,\qquad
    A_H=(1-\theta)a_r>0,\qquad A_J=1/d>0,
    \label{eq:score_positive_coefficients}
\end{equation}
with
\[
    \eta_S=-\frac{1-\theta}{3}\sum_{\{i,j\}\subset S}\rho_{ij},\qquad
    |\eta_S|\le
    \tau_S:=\frac{1-\theta}{3}\sum_{\{i,j\}\subset S}|\rho_{ij}|.
\]
To see the coefficient of quality explicitly, the average of
$(e_i-e_*)+(e_j-e_*)$ over the three pairs is $2(e_S-e_*)$.
The average $R^t_{ij}$ is $1-H_S$, and $e_S=1-\bar q_S$.
These substitutions give \eqref{eq:score_positive_coefficients}.
Thus the local directions of all three terms agree with the signs used
in the selection objective, with a second-order remainder in an interior
neighborhood.

\paragraph{Normalization and fixed objective weights.}
Let $F_t$ be the score computed using target-task qualities and target-task
HI, with the base-item normalization in \eqref{eq:z_score}. For size three,
it has the exact affine form
\begin{equation}
    F_t(S)=c_F+w_q\bar q_S+w_H H_S+w_J J_S,\qquad
    w_q=\frac{\sqrt3}{\sigma_q},\quad
    w_H=\frac{\lambda_1\sqrt3}{\sigma_H},\quad
    w_J=\frac{\lambda_2\sqrt3}{\sigma_J}.
    \label{eq:score_target_affine}
\end{equation}
Here each standard deviation is over its corresponding candidate-pool base
items, and $c_F$ collects their means. A constant channel has coefficient
zero; nonconstant channels satisfy the active-normalization conditions in
\Cref{app:selection_stability}.
Choose any scale $\alpha>0$ and define
\[
    \mathbf A=(A_q,A_H,A_J),\quad
    \mathbf w=(w_q,w_H,w_J),\quad
    X_S=(\bar q_S,H_S,J_S),\quad
    \mathbf b=\mathbf A-\alpha\mathbf w.
\]
For $w_q>0$, taking $\alpha=A_q/w_q$ matches the quality coefficient exactly.
The scale is used only to compare the two mathematical quantities; the
selector retains its original weights.
For two teams define the one-sided coefficient correction
\begin{equation}
    E_{\rm wt}(S,T)=
       \bigl[-\mathbf b^\top(X_S-X_T)\bigr]_+,\qquad [z]_+=\max\{z,0\}.
    \label{eq:score_weight_difference}
\end{equation}
Subtracting \eqref{eq:score_collision_expansion} and using
\eqref{eq:score_target_affine}, together with
$\mathbf b^\top(X_S-X_T)\ge-E_{\rm wt}(S,T)$, gives
\begin{equation}
    K_T-K_S
    \ge\alpha[F_t(S)-F_t(T)]
           -E_{\rm wt}(S,T)-\tau_S-\tau_T.
    \label{eq:score_target_collision_bound}
\end{equation}
This form keeps the specified values of $\lambda_1,\lambda_2$ and measures
their effect through the explicit coefficient differences $\mathbf b$.

\paragraph{Pooled development profiles and the returned team.}
Write $\widehat F$ for the actual score based on estimated target qualities
and pooled HI. Let $(u_t,v_t)$ and $(\widehat u,\widehat v)$ be the corresponding
standardized signals from \eqref{eq:selection_standardized_signals}, with
$R^t,J^t$ used for $v_t$. Put
$\delta u=\widehat u-u_t$ and $\delta v=\widehat v-v_t$.
For two teams with overlap $o=|S\cap T|$, \eqref{eq:selection_overlap_bound}
specializes to
\begin{equation}
\begin{aligned}
    B_{\rm prof}(S,T)
       &:=\sqrt{2-\frac{2o}{3}}\|\delta u\|_2
           +\sqrt{2-\frac{2\binom{o}{2}}{3}}\|\delta v\|_2,\\
    \left|[\widehat F(S)-\widehat F(T)]
                   -[F_t(S)-F_t(T)]\right|
       &\le B_{\rm prof}(S,T).
\end{aligned}
    \label{eq:score_profile_difference}
\end{equation}
Every normalization is recomputed on its own profile, so this expression
includes changes in both base-item values and normalization statistics.

\begin{proposition}[Selection-score condition for an accuracy gain]
\label{prop:score_accuracy_gain}
Under \eqref{eq:joint_model} and the active-normalization conditions above,
let $S_{\rm g}$ be the output of multi-start greedy selection on $\widehat F$,
and let $S_Q$ be the Quality-Only team from the same candidate pool.
For either Choice-Soft or PoE,
\begin{equation}
\begin{aligned}
    \operatorname{Acc}(S_{\rm g})-\operatorname{Acc}(S_Q)
    \ge{}&
       \alpha[\widehat F(S_{\rm g})-\widehat F(S_Q)]
       -\varepsilon_Q\\
       &-\alpha B_{\rm prof}(S_{\rm g},S_Q)
       -E_{\rm wt}(S_{\rm g},S_Q)-\tau_{S_{\rm g}}-\tau_{S_Q}.
\end{aligned}
    \label{eq:score_accuracy_gain}
\end{equation}
In particular, a score advantage exceeding the displayed correction terms
guarantees a strict accuracy improvement.
\end{proposition}

\begin{proof}
The proof of \Cref{prop:quality_only_gain} gives
$\operatorname{Acc}(S_{\rm g})-\operatorname{Acc}(S_Q)
\ge K_Q-K_{S_{\rm g}}-\varepsilon_Q$.
Apply \eqref{eq:score_target_collision_bound} with $S=S_{\rm g}$ and $T=S_Q$.
Then \eqref{eq:score_profile_difference} implies
$F_t(S_{\rm g})-F_t(S_Q)
\ge\widehat F(S_{\rm g})-\widehat F(S_Q)-B_{\rm prof}(S_{\rm g},S_Q)$.
Substitution proves \eqref{eq:score_accuracy_gain}.
\end{proof}

To make the role of multi-start search explicit, let $\mathcal C$ be its set
of distinct terminal teams. Algorithm~\ref{alg:greedy} satisfies
$\widehat F(S_{\rm g})=\max_{T\in\mathcal C}\widehat F(T)$.
Set
\[
    E_{\max}=\max_{T\in\mathcal C}
       \{\alpha B_{\rm prof}(T,S_Q)+E_{\rm wt}(T,S_Q)+\tau_T+\tau_{S_Q}\}.
\]
Therefore the sufficient condition
\begin{equation}
    \alpha\left[\max_{T\in\mathcal C}\widehat F(T)-\widehat F(S_Q)\right]
        >\varepsilon_Q+E_{\max}
    \label{eq:score_terminal_condition}
\end{equation}
ensures that the returned team outperforms Quality-Only. The maximization
here is over the teams actually generated by the seed runs.

Finally, the HI component of the profile change is exactly
\[
    \delta v=
       \lambda_1[Z(1-\widehat{\bar R})-Z(1-R^t)]
       +\lambda_2[Z(\widehat{\bar J})-Z(J^t)].
\]
With exact target qualities, $\delta u=0$. Positive-affine agreement between
each pooled channel and its target counterpart makes $\delta v=0$ by
\eqref{eq:pooling_positive_affine}. More generally,
\eqref{eq:pooling_direction_distance} and \eqref{eq:pooling_edge_perturbation}
bound its norm through the standardized profile directions.
These terms connect target-task accuracy to the pooled signals used by the
selector, while \eqref{eq:score_weight_difference} and
\eqref{eq:score_df_remainder} account separately for the objective weights
and the curvature of the correctness association.

\subsection{Member-Specific Information in Regularized Stacking}
\label{app:stacking_theory}

The concatenated probability features used by \textsc{Stacking} contain both the team's mean
prediction and differences between its members. We analyze the additional reduction in regularized
cross-entropy available from these member-specific differences after the mean prediction is already
available to a linear softmax classifier.

\paragraph{Setting and comparison class.}
Fix a team $S$ of size $k\geq 2$ and relabel its members as $1,\ldots,k$. Let $c\geq 2$ be the
number of classes, and let $p_i(X)\in\Delta^{c-1}$ be each member's aligned choice distribution.
All expectations in this subsection are taken under one fixed distribution $D$ over $(X,Y)$.
This may be a population distribution or the empirical distribution of a development set.
Define the cross-entropy loss on logits using natural logarithms:
\[
    \ell(z,y)=\log\!\left(\sum_{a=1}^{c}\exp(z_a)\right)-z_y.
\]
For $\beta>0$, the optimal value of the Stacking objective is
\begin{equation}
    F_{\beta}(S)
    =\inf_{W_1,\ldots,W_k,b}
    \left\{
        \mathbb{E}_{D}\ell\!\left(\sum_{i=1}^{k}W_i^{\top}p_i(X)+b,Y\right)
        +\frac{\beta}{2}\sum_{i=1}^{k}\|W_i\|_{F}^{2}
    \right\},
    \label{eq:stacking_theory_full_risk}
\end{equation}
where $W_i\in\mathbb{R}^{c\times c}$ are the blocks of the weight matrix in
\Cref{eq:stacking_features,eq:stacking_objective}, and $b\in\mathbb{R}^{c}$ is unpenalized.
For an empirical $D$, this is exactly the average cross-entropy plus the $L_2$ penalty in
\Cref{eq:stacking_objective}. In the implementation, the data-fit gradient is averaged over the
development examples and the penalty contributes $\mathtt{l2}\,W$ to the weight gradient, so the
coefficient here is $\beta=\mathtt{l2}$.

Let $\bar p(X)=k^{-1}\sum_{i=1}^{k}p_i(X)$. The comparison class uses only this mean probability
vector as the input to a linear softmax classifier:
\begin{equation}
    G_{\beta/k}(\bar p)
    =\min_{M,b}
    \left\{
        \mathbb{E}_{D}\ell(M^{\top}\bar p(X)+b,Y)
        +\frac{\beta}{2k}\|M\|_{F}^{2}
    \right\}.
    \label{eq:stacking_theory_mean_risk}
\end{equation}
This is a restricted class within Stacking: setting $W_i=M/k$ reproduces its logits and penalty.
Thus the factor $\beta/k$ follows from the original regularizer. The comparison class is used only
for analysis; its input equals the \textsc{Choice-Soft} mean, followed by a learned linear softmax
map. Assume that \Cref{eq:stacking_theory_mean_risk} has a finite minimizer $(M_0,b_0)$.
For example, positive probability for every class is sufficient when $\beta>0$, after fixing
$\mathbf{1}^{\top}b=0$ to remove the irrelevant common shift in the logits.

Define the mean classifier, the concatenated member differences, and their residual alignment by
\begin{equation}
    \begin{aligned}
        \pi_0(X)&=\operatorname{softmax}(M_0^{\top}\bar p(X)+b_0),\\
        \delta\phi(X)&=[p_1(X)-\bar p(X);\ldots;p_k(X)-\bar p(X)],\\
        \Gamma&=\mathbb{E}_{D}\!\left[\delta\phi(X)(e_Y-\pi_0(X))^{\top}\right],
    \end{aligned}
    \label{eq:stacking_theory_residual_alignment}
\end{equation}
where $e_Y$ is the one-hot vector of the true class. Also set
\begin{equation}
    \Sigma_{\delta}=\mathbb{E}_{D}[\delta\phi(X)\delta\phi(X)^{\top}],
    \qquad
    L_{\delta}=\beta+\frac12\lambda_{\max}(\Sigma_{\delta}).
    \label{eq:stacking_theory_curvature}
\end{equation}
All these moments are finite because the probability features are bounded. To use the same
distribution throughout the analysis, write
\[
    H_D(S)=\mathrm{HI}_{dist}(S;J^D),\qquad
    J^D_{ij}=\mathbb{E}_{D}\!\left[\operatorname{JSD}_{2}(p_i(X),p_j(X))\right],
\]
where $\operatorname{JSD}_{2}$ uses logarithms to base $2$ and $H_D$ follows the pairwise averaging
rule in \Cref{eq:hi_dist}.

\begin{theorem}[Additional information available to regularized Stacking]
\label{thm:stacking_member_information}
Under the setting above, let
$\Delta_S=G_{\beta/k}(\bar p)-F_{\beta}(S)$. Then
\begin{equation}
    \frac{\|\Gamma\|_{F}^{2}}{2L_{\delta}}
    \;\leq\;\Delta_S
    \;\leq\;\frac{\|\Gamma\|_{F}^{2}}{2\beta}
    \;\leq\;\frac{2\ln 2\,(k-1)}{\beta}\,H_D(S).
    \label{eq:stacking_theory_gain_bounds}
\end{equation}
Consequently, $\Delta_S>0$ if and only if $\Gamma\neq 0$.
\end{theorem}

\begin{proof}
\emph{Step 1: separate the mean and member-specific weights.}
For arbitrary Stacking weights, define
\[
    M=\sum_{i=1}^{k}W_i,\qquad
    U_i=W_i-\frac{M}{k},\qquad
    U=[U_1;\ldots;U_k].
\]
Then $\sum_iU_i=0$, and every collection of weights admits this decomposition. Conversely, any
$M$ and $U$ satisfying $\sum_iU_i=0$ define weights $W_i=M/k+U_i$. The logits and penalty obey
\begin{align}
    \sum_iW_i^{\top}p_i+b
        &=M^{\top}\bar p+U^{\top}\delta\phi+b,
        \label{eq:stacking_theory_logit_decomposition}\\
    \sum_i\|W_i\|_{F}^{2}
        &=\frac{\|M\|_{F}^{2}}{k}+\|U\|_{F}^{2}.
        \label{eq:stacking_theory_penalty_decomposition}
\end{align}
The first identity uses $\sum_iU_i=0$; the cross terms in the second identity vanish for the
same reason. Therefore, with
\[
    \Psi(M,b,U)=
    \mathbb{E}_{D}\ell(M^{\top}\bar p+U^{\top}\delta\phi+b,Y)
    +\frac{\beta}{2k}\|M\|_{F}^{2}+\frac{\beta}{2}\|U\|_{F}^{2},
\]
we have $F_{\beta}(S)=\inf_{M,b,U:\,\sum_iU_i=0}\Psi(M,b,U)$ and
$G_{\beta/k}(\bar p)=\Psi(M_0,b_0,0)$. In particular, $\Delta_S\geq 0$.

\emph{Step 2: bound the improvement using the residual gradient.}
The finite optimum of the mean classifier satisfies
\begin{equation}
    \mathbb{E}_{D}[\bar p(\pi_0-e_Y)^{\top}]+\frac{\beta}{k}M_0=0,
    \qquad
    \mathbb{E}_{D}[\pi_0-e_Y]=0.
    \label{eq:stacking_theory_first_order}
\end{equation}
If $\mathbf{1}^{\top}b=0$ is imposed, stationarity initially gives the second equation on that
subspace. The bias gradient always has coordinate sum zero, so this also makes its full gradient
zero. Differentiation under the expectation is justified by bounded probability features and
bounded first derivatives of cross-entropy with respect to logits.

Apply convexity to the cross-entropy term at $(M_0,b_0,0)$, and expand the quadratic penalties.
The first-order terms in $M-M_0$ and $b-b_0$ cancel by
\Cref{eq:stacking_theory_first_order}, leaving
\begin{equation}
    \Psi(M,b,U)\geq G_{\beta/k}(\bar p)-\langle\Gamma,U\rangle_F
    +\frac{\beta}{2k}\|M-M_0\|_{F}^{2}
    +\frac{\beta}{2}\|U\|_{F}^{2}.
    \label{eq:stacking_theory_convex_lower_bound}
\end{equation}
Here $\langle A,B\rangle_F=\operatorname{tr}(A^{\top}B)$. Completing the square in $U$ gives
\[
    -\langle\Gamma,U\rangle_F+\frac{\beta}{2}\|U\|_{F}^{2}
    =\frac{\beta}{2}\left\|U-\frac{\Gamma}{\beta}\right\|_{F}^{2}
     -\frac{\|\Gamma\|_{F}^{2}}{2\beta}.
\]
Taking the infimum in \Cref{eq:stacking_theory_convex_lower_bound} yields
$F_{\beta}(S)\geq G_{\beta/k}(\bar p)-\|\Gamma\|_{F}^{2}/(2\beta)$, proving the upper
bound on $\Delta_S$. This argument uses the quadratic penalty on the weights and the bias
stationarity condition; it requires no strong convexity in the unpenalized bias.

\emph{Step 3: exhibit a feasible improvement.}
For a probability vector $\pi$, the Hessian of cross-entropy with respect to logits is
$B(\pi)=\operatorname{diag}(\pi)-\pi\pi^{\top}$. For every $v\in\mathbb{R}^{c}$,
\begin{align*}
    v^{\top}B(\pi)v
        &=\sum_{a<b}\pi_a\pi_b(v_a-v_b)^2\\
        &\leq 2\sum_a\pi_a(1-\pi_a)v_a^2
        \leq\frac12\|v\|_2^2.
\end{align*}
The last inequality uses $\pi_a(1-\pi_a)\leq 1/4$, so
$0\preceq B(\pi)\preceq I/2$. Fix $(M,b)=(M_0,b_0)$ and write
$f(U)=\Psi(M_0,b_0,U)$. For any matrix direction $V$, its second directional derivative obeys
\begin{align*}
    D^2f(U)[V,V]
        &\leq\beta\|V\|_{F}^{2}
          +\frac12\mathbb{E}_{D}\|V^{\top}\delta\phi\|_2^2\\
        &=\beta\|V\|_{F}^{2}
          +\frac12\operatorname{tr}(V^{\top}\Sigma_{\delta}V)
        \leq L_{\delta}\|V\|_{F}^{2}.
\end{align*}
Since $\nabla f(0)=-\Gamma$, this global curvature bound implies
\[
    f(U)\leq G_{\beta/k}(\bar p)-\langle\Gamma,U\rangle_F
              +\frac{L_{\delta}}{2}\|U\|_{F}^{2}.
\]
Partition $\Gamma$ into blocks $\Gamma_i\in\mathbb{R}^{c\times c}$. Their sum is zero:
\[
    \sum_i\Gamma_i
    =\mathbb{E}_{D}\!\left[
        \left(\sum_i(p_i-\bar p)\right)(e_Y-\pi_0)^{\top}
      \right]=0.
\]
Thus $U=\Gamma/L_{\delta}$ satisfies the constraint $\sum_iU_i=0$ and gives
\[
    F_{\beta}(S)\leq f(\Gamma/L_{\delta})
    \leq G_{\beta/k}(\bar p)-\frac{\|\Gamma\|_{F}^{2}}{2L_{\delta}}.
\]
This proves the lower bound on $\Delta_S$.

\emph{Step 4: relate the residual gradient to pairwise JSD.}
First, $\|e_Y-\pi_0\|_2^2\leq 2$. The triangle inequality and Cauchy--Schwarz therefore give
\begin{equation}
    \|\Gamma\|_{F}^{2}
    \leq\mathbb{E}_{D}\|\delta\phi\|_2^2\,
         \mathbb{E}_{D}\|e_Y-\pi_0\|_2^2
    \leq 2\mathbb{E}_{D}\|\delta\phi\|_2^2.
    \label{eq:stacking_theory_gradient_energy}
\end{equation}
For completeness, the constants connecting this energy to JSD follow directly from Pinsker's
inequality. For probability vectors $p,q$, put $m=(p+q)/2$ and use natural logarithms for
$\operatorname{KL}$ and $\operatorname{JSD}_{\mathrm{nat}}$. Then
\begin{align*}
    \operatorname{JSD}_{\mathrm{nat}}(p,q)
        &=\frac12\operatorname{KL}(p\|m)+\frac12\operatorname{KL}(q\|m)\\
        &\geq\frac14\|p-m\|_1^2+\frac14\|q-m\|_1^2
         =\frac18\|p-q\|_1^2.
\end{align*}
Because $p-q$ has coordinate sum zero, its positive and negative parts have equal total mass
$a$. Hence $\|p-q\|_2^2\leq 2a^2=\|p-q\|_1^2/2$. Converting to bits yields
\begin{equation}
    \operatorname{JSD}_{2}(p,q)
    =\frac{\operatorname{JSD}_{\mathrm{nat}}(p,q)}{\ln2}
    \geq\frac{\|p-q\|_2^2}{4\ln2}.
    \label{eq:stacking_theory_jsd_energy}
\end{equation}
Finally, expanding the squared norms around the mean gives
\begin{align*}
    \mathbb{E}_{D}\|\delta\phi\|_2^2
        &=\mathbb{E}_{D}\sum_i\|p_i-\bar p\|_2^2
         =\frac1k\sum_{i<j}\mathbb{E}_{D}\|p_i-p_j\|_2^2\\
        &\leq\frac{4\ln2}{k}\sum_{i<j}J^D_{ij}
         =2\ln2\,(k-1)H_D(S).
\end{align*}
Together with \Cref{eq:stacking_theory_gradient_energy}, this gives
$\|\Gamma\|_{F}^{2}\leq4\ln2\,(k-1)H_D(S)$ and proves the final inequality in
\Cref{eq:stacking_theory_gain_bounds}. If $\Gamma=0$, the upper bound and $\Delta_S\geq0$
give $\Delta_S=0$. If $\Gamma\neq0$, the lower bound is strictly positive because
$0<\beta\leq L_{\delta}<\infty$.
\end{proof}

\paragraph{Interpretation.}
Pairwise JSD quantifies the size of the probability differences available to Stacking. The matrix
$\Gamma$ identifies whether those differences align with the predictive residual left by the mean
classifier. The theorem concerns the additional value of member-specific features beyond that
classifier, measured by the optimal regularized cross-entropy under $D$. Improvements already
present in the mean probability vector are part of the comparison class. The theorem
complements the soft-aggregation analysis by characterizing the gain from retaining
member identity in a learned linear combiner.

\section{Implementation Details}
\label{app:implementation}

\subsection{Choice-Space Aligned Scoring Implementation}
\label{app:scoring_implementation}

For each example $x$ with label space $\mathcal{Y}_t = \{y_1, \ldots, y_{|\mathcal{Y}_t|}\}$, we:

\begin{enumerate}
    \item Force the model to generate each choice token $y_k$ (e.g., ``A'', ``B'', ``C'', ``D'')
    \item Extract the log-probability $\ell_i(y_k | x)$ from the model's output distribution
    \item Normalize across $\mathcal{Y}_t$:
    \begin{equation}
        p_i(y_k | x) = \frac{\exp(\ell_i(y_k | x))}{\sum_{y' \in \mathcal{Y}_t} \exp(\ell_i(y' | x))}
    \end{equation}
\end{enumerate}

This ensures $\sum_{y \in \mathcal{Y}_t} p_i(y | x) = 1$ with negligible OTHER mass.

\paragraph{Prompt Format.}
We use a chat-based prompt format with two messages:

\definecolor{c2mas}{RGB}{128,192,192}
\begin{tcolorbox}[title={Profiling Prompt Template}, 
    colframe=c2mas, colback=white, breakable]
\textbf{User Message:}
\begin{lstlisting}[basicstyle=	\ttfamily\small]
[Question]

Options:
[Option 1]
[Option 2]
...

Complete the sentence 'The correct answer is ' by outputting 
exactly one letter from: A, B, C, D. Output nothing else.
\end{lstlisting}

\textbf{Assistant Message (prefix):} (Note: there is a trailing \textbf{space} after ``is''.)
\begin{lstlisting}[basicstyle=\ttfamily\small]
The correct answer is 
\end{lstlisting}

\end{tcolorbox} 

\paragraph{Why normalize within the choice space.}
Outside-choice token mass can introduce divergence unrelated to preferences among the actual
answer options. The following entropy identity \citep{lin1991divergence} isolates this effect.

\begin{lemma}[Divergence due to outside-choice mass]
\label{prop:other_dominance}
Let $\mathcal{Y}$ be the evaluation label space (e.g., $\{A, B, C, D\}$) and $\mathcal{Y}' = \mathcal{Y} \cup \{\text{OTHER}\}$ be an extended space. Suppose two agents have distributions:
\begin{align}
    p_i &= (1 - \alpha_i) q + \alpha_i \delta_{\text{OTHER}} \\
    p_j &= (1 - \alpha_j) q + \alpha_j \delta_{\text{OTHER}}
\end{align}
where $q$ is a distribution over $\mathcal{Y}$ (identical for both agents), $\delta_{\text{OTHER}}$ is a point mass on OTHER, and $\alpha_i, \alpha_j \in [0, 1]$.

Then:
\begin{equation}
    \JSD(p_i, p_j) = \JSD\big([\alpha_i, 1-\alpha_i], [\alpha_j, 1-\alpha_j]\big)
\end{equation}

Thus all divergence arises from the difference in OTHER mass, although their common choice distribution is $q$ (the conditional distribution whenever $\alpha_i,\alpha_j<1$).
\end{lemma}

\begin{proof}
Let $\bar{\alpha}=(\alpha_i+\alpha_j)/2$, and let $H$ denote Shannon entropy in bits.
Since $q$ and $\delta_{\text{OTHER}}$ have disjoint supports,
\begin{align}
    H(p_i) &= H(\alpha_i)+(1-\alpha_i)H(q), \\
    H(p_j) &= H(\alpha_j)+(1-\alpha_j)H(q), \\
    H\!\left(\frac{p_i+p_j}{2}\right)
        &= H(\bar{\alpha})+(1-\bar{\alpha})H(q).
\end{align}
Substituting these identities into the entropy expression for JSD cancels the $H(q)$ terms:
\begin{equation}
    \JSD(p_i,p_j)
    = H\left(\frac{\alpha_i + \alpha_j}{2}\right) - \frac{H(\alpha_i) + H(\alpha_j)}{2}
    = \JSD\big([\alpha_i,1-\alpha_i],[\alpha_j,1-\alpha_j]\big).
\end{equation}
where $H(p) = -p \log_2 p - (1-p) \log_2(1-p)$ is the binary entropy.

For $\alpha_i=0.1$ and $\alpha_j=0.9$, the divergence is
$1-H(0.1)\approx 0.531$ bits, entirely due to the difference in OTHER mass.
\end{proof}

Choice-space normalization removes this outside-mass contribution before computing $J^t$.

\subsection{Paired Bootstrap Procedure}
\label{app:bootstrap}

For comparing two methods $A$ and $B$ on a test set of $N$ examples:

\begin{enumerate}
    \item Compute per-example correctness: $c_A(x_n), c_B(x_n) \in \{0, 1\}$ for $n = 1, \ldots, N$
    \item For $b = 1, \ldots, B$ (we use $B = 2000$):
    \begin{enumerate}
        \item Sample indices $\{i_1, \ldots, i_N\}$ uniformly with replacement from $\{1, \ldots, N\}$
        \item Compute bootstrap accuracies:
        \begin{align}
            \text{Acc}_A^{(b)} &= \frac{1}{N} \sum_{n=1}^N c_A(x_{i_n}) \\
            \text{Acc}_B^{(b)} &= \frac{1}{N} \sum_{n=1}^N c_B(x_{i_n})
        \end{align}
        \item Compute bootstrap difference: $\Delta^{(b)} = \text{Acc}_A^{(b)} - \text{Acc}_B^{(b)}$
    \end{enumerate}
    \item Compute 95\% confidence interval: $[\text{percentile}(\{\Delta^{(b)}\}, 2.5), \text{percentile}(\{\Delta^{(b)}\}, 97.5)]$
    \item Mark as significant if the CI excludes 0
\end{enumerate}

The paired structure (using the same bootstrap sample for both methods) preserves the correlation between $c_A$ and $c_B$, providing more statistical power than unpaired bootstrap.

\subsection{Hyperparameter Selection}
\label{app:hyperparameter_selection}

We maintain strict development--test separation when selecting the objective weights. We first
perform a grid search using only development splits to identify a stable high-performing region
for the two heterogeneity weights, and then fix $(\lambda_1,\lambda_2)=(0.13,0.05)$ for all
datasets and aggregation rules. No test-set performance is used for per-task hyperparameter
tuning.

To examine whether this choice reflects a task-specific artifact, we further conduct
leave-one-task-out hyperparameter selection over the seven primary benchmarks. For each held-out
benchmark, the weight pair is selected using the development splits of the remaining six
benchmarks and then applied to the held-out benchmark. Across the held-out tasks, the grid contains
broad high-performing plateaus rather than isolated optima: multiple neighboring weight pairs
achieve the same development-set performance. The fixed weight pair used in our main experiments
lies within this high-performing region for 6 out of 7 held-out tasks. This indicates that the
selected weights capture a stable quality--heterogeneity trade-off rather than a task-specific
tuning artifact.

\subsection{Computational Cost of Team Selection}
\label{app:selection_cost}

With $m$ candidates and target size $k$, Algorithm~\ref{alg:greedy} tries $m$ seeds and
at most $k-1$ additions per seed. If within-team quality and pairwise sums are cached,
evaluating an extension from a size-$s$ team needs $s$ pairwise lookups plus constant-time
quality and normalization operations. Evaluating all remaining candidates therefore costs
$O(ms)$ per round and $O(m\sum_{s=1}^{k-1}s)=O(mk^2)$ per seed. The total cost is
$O(m^2k^2)$, with $O(m^2)$ storage for the pairwise matrices. These are operation counts for
cached score evaluation. At the experimental size $m=11$, $k=3$, exhaustive enumeration is also
feasible; multi-start search extends the same scoring rule to larger pools.

\subsection{Aggregation Rules}
\label{app:aggregation_rules}

\paragraph{Choice-Soft.}
\textsc{Choice-Soft} aggregates by averaging the members' aligned choice distributions. For a team
$S_t$ and instance $x$, we compute the mean distribution as follows:
\begin{equation}
    p_{\textsc{cs}}^t(y\mid x)
    = \frac{1}{|S_t|}\sum_{i\in S_t} p_i^t(y\mid x).
    \label{eq:choice_soft}
\end{equation}
We then predict $\hat{y}_{S_t}(x)=\arg\max_{y\in\mathcal{Y}_t} p_{\textsc{cs}}^t(y\mid x)$.

\paragraph{PoE.}
\textsc{PoE} (Product of Experts) combines distributions multiplicatively, rewarding labels that
receive consistently high probability across members. We form the aggregated distribution as:
\begin{equation}
    p_{\textsc{poe}}^t(y\mid x)
    = \frac{\prod_{i\in S_t} p_i^t(y\mid x)}{\sum_{y'\in\mathcal{Y}_t}\prod_{i\in S_t} p_i^t(y'\mid x)},
    \label{eq:poe}
\end{equation}
and predict $\hat{y}_{S_t}(x)=\arg\max_{y\in\mathcal{Y}_t} p_{\textsc{poe}}^t(y\mid x)$.

\paragraph{DS.}
\textsc{DS} (Dawid--Skene) models each member as a noisy labeler with an agent-specific confusion matrix estimated on $D_t^{\mathrm{dev}}$. Let $\hat{y}_i(x)=\arg\max_{y\in\mathcal{Y}_t} p_i^t(y\mid x)$ denote the label predicted by model $i$. We estimate a prior $\pi^t(y)$ and, for each model $i$, a confusion matrix $A_i^t\in\mathbb{R}^{|\mathcal{Y}_t|\times|\mathcal{Y}_t|}$ via additive smoothing (with a small constant $\epsilon$), where $A_i^t(a,b)$ estimates $\mathbb{P}(\hat{y}_i=b\mid y=a)$.
At inference time, we compute the posterior up to proportionality:
\begin{equation}
    p_{\textsc{ds}}^t(y\mid x)\;\propto\;\pi^t(y)\prod_{i\in S_t} A_i^t\!\left(y,\,\hat{y}_i(x)\right).
    \label{eq:ds}
\end{equation}
We then renormalize over $\mathcal{Y}_t$ and predict
$\hat{y}_{S_t}(x)=\arg\max_{y\in\mathcal{Y}_t} p_{\textsc{ds}}^t(y\mid x)$.

\paragraph{Stacking.}
\textsc{Stacking} trains a lightweight combiner on $D_t^{\mathrm{dev}}$. For a team $S_t$ and instance
$x$, we concatenate members' choice distributions:
\begin{equation}
    \phi(x) = \big[p_{i_1}^t(\cdot\mid x);\ldots;p_{i_k}^t(\cdot\mid x)\big]
    \in\mathbb{R}^{k|\mathcal{Y}_t|}.
    \label{eq:stacking_features}
\end{equation}
We learn a linear classifier parametrized by $W\in\mathbb{R}^{k|\mathcal{Y}_t|\times|\mathcal{Y}_t|}$ and $b\in\mathbb{R}^{|\mathcal{Y}_t|}$ via $L_2$-regularized cross-entropy. Let $\pi_{W,b}(y\mid x)=\mathrm{softmax}(W^\top\phi(x)+b)_y$. We train the stacking combiner by minimizing:
\begin{equation}
\begin{aligned}
    \min_{W,b}\; \mathcal{L}(W,b)
    ={}& \frac{1}{|D_t^{\mathrm{dev}}|}
    \sum_{(x,y)\in D_t^{\mathrm{dev}}}
    -\log \pi_{W,b}(y\mid x) \\
    &+ \frac{\beta}{2}\|W\|_F^2 .
\end{aligned}
\label{eq:stacking_objective}
\end{equation}
At test time, we predict $\hat{y}_{S_t}(x)=\arg\max_{y\in\mathcal{Y}_t} \pi_{W,b}(y\mid x)$.

\section{Supplementary Results}
\label{app:additional_experiments}

This section provides additional results that complement the main experiments. We first report the
complete benchmark tables omitted from the main text for space, then analyze the team compositions
selected by \cmas{}, and finally provide robustness and ablation analyses for team size and
objective validity.

\subsection{Complete Benchmark Results}
\label{app:complete_results}

\Cref{tab:complete_results} reports the full per-benchmark accuracy across all 11 open-weight candidate models, the Self-Consistency baseline, and \cmas{} under four aggregation rules, separated into (a) 7 primary benchmarks and (b) 6 OOD benchmarks.

\begin{table}[htbp]
    \centering
    \refstepcounter{table}
    \label{tab:complete_results}
    \caption*{Table~\ref{tab:complete_results}: Complete per-benchmark results across all 13 benchmarks. We report the full open-weight candidate pool ($m=11$), the single-model Self-Consistency baseline, and \cmas{} under four aggregation rules. (a) covers the seven primary benchmarks used in \Cref{tab:main_results}; (b) covers the six OOD benchmarks, where heterogeneity matrices are pooled only on the primary set and applied to OOD teams (dev$\rightarrow$test). \textbf{Avg.} reports the average accuracy within each panel. Abbreviations: ARC-C: ARC-Challenge; CSQA: CommonsenseQA; OBQA: OpenBookQA.}
    \begingroup
    \small
    \setlength{\tabcolsep}{0.8pt}
    \renewcommand{\arraystretch}{1.08}

        \caption*{(a) Primary benchmarks (7).}
        \centering
            \begin{tabular*}{\linewidth}{@{\extracolsep{\fill}}l*{8}{l}@{}}
                \toprule[1.25pt]
                \multirow[c]{2}{*}{\textbf{Method}}
                & \multicolumn{7}{c}{\textbf{Dataset}}
                & \multirow[c]{2}{*}{\textbf{Avg.}} \\
                & \textbf{ARC-C} & \textbf{CSQA} & \textbf{LogiQA2} & \textbf{MedQA} & \textbf{MMLU} & \textbf{MMLU-Pro} & \textbf{OBQA} & \\
                \midrule[1.1pt]
                \rowcolor[rgb]{0.93,0.93,0.93}
                \multicolumn{9}{c}{\textbf{Open-Weight Individual Models}} \\
                \texttt{Llama-3.1-8B} & $80.15$ & $75.28$ & $54.03$ & $61.98$ & $67.70$ & $37.92$ & $81.74$ & $65.54$ \\
                \texttt{Qwen3-8B} & $89.61$ & $79.78$ & $67.23$ & $61.05$ & $71.72$ & $45.79$ & $83.24$ & $71.20$ \\
                \texttt{Gemma-2-9B} & $88.48$ & $78.56$ & $63.20$ & $59.74$ & $74.44$ & $44.01$ & $85.49$ & $70.56$ \\
                \texttt{Ministral-3-8B} & $79.40$ & $62.08$ & $59.93$ & $55.52$ & $67.88$ & $41.20$ & $68.45$ & $62.07$ \\
                \texttt{Granite-3.3-8B} & $80.90$ & $72.66$ & $53.00$ & $51.22$ & $65.45$ & $35.02$ & $80.52$ & $62.68$ \\
                \texttt{GLM-4-9B} & $85.86$ & $74.44$ & $56.55$ & $55.81$ & $64.61$ & $31.84$ & $84.93$ & $64.86$ \\
                \texttt{Nemotron-Nano-9B} & $88.95$ & $74.72$ & $65.07$ & $60.21$ & $72.94$ & $45.88$ & $88.95$ & $70.96$ \\
                \texttt{EuroLLM-9B} & $75.94$ & $72.85$ & $45.88$ & $50.00$ & $59.93$ & $26.97$ & $73.88$ & $57.92$ \\
                \texttt{OLMo-3-7B} & $72.47$ & $71.54$ & $42.88$ & $41.48$ & $56.18$ & $32.87$ & $66.57$ & $54.86$ \\
                \texttt{Apertus-8B} & $72.00$ & $65.54$ & $46.82$ & $48.22$ & $55.71$ & $29.68$ & $68.63$ & $55.23$ \\
                \texttt{InternLM3-8B} & $88.20$ & $75.94$ & $67.13$ & $59.64$ & $71.16$ & $42.70$ & $84.36$ & $69.88$ \\
                \midrule
                \rowcolor[rgb]{0.93,0.93,0.93}
                \multicolumn{9}{c}{\textbf{Single-Model Aggregation}} \\
                Self-Consistency & $89.14$ & $78.65$ & $67.32$ & $61.70$ & $73.41$ & $44.76$ & $88.58$ & $71.94$ \\
                \midrule
                \rowcolor[rgb]{0.93,0.93,0.93}
                \multicolumn{9}{c}{\textbf{\cmas{} Team Selection}} \\
                \cmas{} (Choice-Soft) & $91.48$ & $83.90$ & $70.88$ & $66.39$ & $76.59$ & $48.50$ & $90.45$ & $75.45$ \\
                \cmas{} (PoE) & $91.57$ & $83.90$ & $70.69$ & $65.36$ & $76.69$ & $48.97$ & $89.51$ & $75.24$ \\
                \cmas{} (DS) & $90.92$ & $82.87$ & $70.32$ & $65.17$ & $75.47$ & $47.10$ & $90.45$ & $74.61$ \\
                \cmas{} (Stacking) & $91.67$ & $84.18$ & $70.60$ & $65.54$ & $76.12$ & $48.97$ & $91.57$ & $75.52$ \\
                \bottomrule[1.25pt]
            \end{tabular*}
    \endgroup
\end{table}

\begin{table}[htbp]
    \centering
    \caption*{Table~\ref{tab:complete_results} (continued).}
    \begingroup
    \small
    \setlength{\tabcolsep}{0.8pt}
    \renewcommand{\arraystretch}{1.08}
        \caption*{(b) OOD benchmarks (6).}
        \centering
            \begin{tabular*}{\linewidth}{@{\extracolsep{\fill}}l*{7}{l}@{}}
                \toprule[1.25pt]
                \multirow[c]{2}{*}{\textbf{Method}}
                & \multicolumn{6}{c}{\textbf{Dataset}}
                & \multirow[c]{2}{*}{\textbf{Avg.}} \\
                & \textbf{AQUA-RAT} & \textbf{C-Eval} & \textbf{MathQA} & \textbf{RACE} & \textbf{ReClor} & \textbf{StrategyQA} & \\
                \midrule[1.1pt]
                \rowcolor[rgb]{0.93,0.93,0.93}
                \multicolumn{8}{c}{\textbf{Open-Weight Individual Models}} \\
                \texttt{Llama-3.1-8B} & $33.99$ & $49.53$ & $33.24$ & $60.77$ & $62.45$ & $67.32$ & $51.22$ \\
                \texttt{Qwen3-8B} & $42.60$ & $75.09$ & $40.07$ & $63.11$ & $81.46$ & $67.88$ & $61.70$ \\
                \texttt{Gemma-2-9B} & $30.90$ & $55.34$ & $32.21$ & $64.51$ & $77.06$ & $70.13$ & $55.03$ \\
                \texttt{Ministral-3-8B} & $31.37$ & $58.80$ & $27.15$ & $58.80$ & $67.79$ & $58.52$ & $50.41$ \\
                \texttt{Granite-3.3-8B} & $32.40$ & $45.51$ & $31.09$ & $60.96$ & $66.48$ & $64.89$ & $50.22$ \\
                \texttt{GLM-4-9B} & $30.81$ & $65.73$ & $30.43$ & $60.39$ & $61.89$ & $67.42$ & $52.78$ \\
                \texttt{Nemotron-Nano-9B} & $45.32$ & $54.21$ & $40.73$ & $64.14$ & $82.12$ & $63.20$ & $58.29$ \\
                \texttt{EuroLLM-9B} & $27.81$ & $46.54$ & $24.91$ & $56.93$ & $56.46$ & $64.33$ & $46.16$ \\
                \texttt{OLMo-3-7B} & $33.61$ & $37.92$ & $30.81$ & $58.71$ & $55.24$ & $63.67$ & $46.66$ \\
                \texttt{Apertus-8B} & $27.06$ & $47.19$ & $28.46$ & $53.75$ & $50.09$ & $61.70$ & $44.71$ \\
                \texttt{InternLM3-8B} & $32.12$ & $84.08$ & $29.31$ & $63.76$ & $72.47$ & $62.55$ & $57.38$ \\
                \midrule
                \rowcolor[rgb]{0.93,0.93,0.93}
                \multicolumn{8}{c}{\textbf{Single-Model Aggregation}} \\
                Self-Consistency & $45.51$ & $81.46$ & $39.23$ & $67.60$ & $81.65$ & $71.25$ & $64.45$ \\
                \midrule
                \rowcolor[rgb]{0.93,0.93,0.93}
                \multicolumn{8}{c}{\textbf{\cmas{} Team Selection}} \\
                \cmas{} (Choice-Soft) & $45.69$ & $80.24$ & $41.95$ & $67.98$ & $84.18$ & $69.76$ & $64.97$ \\
                \cmas{} (PoE) & $45.22$ & $78.93$ & $41.57$ & $67.79$ & $84.55$ & $69.66$ & $64.62$ \\
                \cmas{} (DS) & $44.19$ & $80.52$ & $42.04$ & $68.35$ & $83.90$ & $70.79$ & $64.97$ \\
                \cmas{} (Stacking) & $46.72$ & $85.96$ & $43.91$ & $68.35$ & $84.83$ & $70.69$ & $66.74$ \\
                \bottomrule[1.25pt]
            \end{tabular*}
    \endgroup
\end{table}

\begin{table}[t]
    \centering
    \caption{OOD test accuracy (\%) on the remaining 6 benchmarks. For \cmas{}, we pool heterogeneity matrices using only the seven primary benchmarks in \Cref{tab:main_results} and apply them to select teams on these 6 benchmarks (dev$\rightarrow$test). $\uparrow$ indicates gain over Random-$k$ (pp). \textbf{Bold} marks the best within each aggregation block.}
    \label{tab:ood_results}
    \begingroup
    \footnotesize
    \setlength{\tabcolsep}{0.8pt}
    \renewcommand{\arraystretch}{1.08}
        \begin{tabular*}{\linewidth}{@{\extracolsep{\fill}}cl*{7}{l}@{}}
            \toprule[1.25pt]
            \multicolumn{2}{c}{\multirow[c]{2}{*}{\textbf{Method}}}
            & \multicolumn{6}{c}{\textbf{Dataset}}
            & \multirow[c]{2}{*}{\textbf{Avg.}} \\
            \multicolumn{2}{c}{}
            & \textbf{AQUA-RAT} & \textbf{C-Eval} & \textbf{MathQA} & \textbf{RACE} & \textbf{ReClor} & \textbf{StrategyQA} & \\
            \midrule[1.1pt]
            \rowcolor[rgb]{0.93,0.93,0.93}
            \multicolumn{9}{c}{\textbf{Closed-Source Models (for reference)}} \\
            \multicolumn{2}{l}{Gemini-2.5-Flash} & $26.12$ & $73.41$ & $30.15$ & $63.86$ & $79.21$ & $62.73$ & $55.91$ \\
            \multicolumn{2}{l}{GPT-4o} & $29.96$ & $72.47$ & $29.49$ & $65.07$ & $84.36$ & $80.06$ & $60.24$ \\
            \midrule
            \rowcolor[rgb]{0.93,0.93,0.93}
            \multicolumn{9}{c}{\textbf{Single-Model Aggregation}} \\
            \multicolumn{2}{l}{Self-Consistency} & $45.51$ & $81.46$ & $39.23$ & $67.60$ & $81.65$ & $71.25$ & $64.45$ \\
            \midrule
            \rowcolor[rgb]{0.93,0.93,0.93}
            \multicolumn{9}{c}{\textbf{Team Selection Baselines}} \\
            \multirow[c]{4}{*}{\rotatebox[origin=c]{90}{\textit{Choice-Soft}}} & Random-$k$ & $39.10$ & $65.71$ & $36.83$ & $66.31$ & $75.55$ & $66.63$ & $58.36$ \\
            & Caruana & $\mathbf{47.00}$ & $\mathbf{84.55}$ & $40.92$ & $67.23$ & $83.71$ & $\mathbf{70.13}$ & $\mathbf{65.59}$ \\
            & Quality-Only & $45.32$ & $80.24$ & $\mathbf{42.88}$ & $\mathbf{67.98}$ & $82.96$ & $69.76$ & $64.86$ \\
            & \textbf{\cmas{} (Ours)} & \scoregain{$45.69$}{6.59} & \scoregain{$80.24$}{14.53} & \scoregain{$41.95$}{5.12} & \scoregain{$\mathbf{67.98}$}{1.67} & \scoregain{$\mathbf{84.18}$}{8.63} & \scoregain{$69.76$}{3.13} & \scoregain{$64.97$}{6.61} \\
            \midrule
            \multirow[c]{4}{*}{\rotatebox[origin=c]{90}{\textit{PoE}}} & Random-$k$ & $39.24$ & $66.63$ & $36.70$ & $66.47$ & $76.38$ & $66.76$ & $58.70$ \\
            & Caruana & $\mathbf{47.00}$ & $\mathbf{83.24}$ & $41.48$ & $\mathbf{68.07}$ & $\mathbf{84.64}$ & $\mathbf{69.94}$ & $\mathbf{65.73}$ \\
            & Quality-Only & $46.72$ & $78.93$ & $\mathbf{42.04}$ & $67.79$ & $84.46$ & $69.66$ & $64.93$ \\
            & \textbf{\cmas{} (Ours)} & \scoregain{$45.22$}{5.98} & \scoregain{$78.93$}{12.30} & \scoregain{$41.57$}{4.87} & \scoregain{$67.79$}{1.32} & \scoregain{$84.55$}{8.17} & \scoregain{$69.66$}{2.90} & \scoregain{$64.62$}{5.92} \\
            \midrule
            \multirow[c]{4}{*}{\rotatebox[origin=c]{90}{\textit{DS}}} & Random-$k$ & $38.71$ & $68.38$ & $36.33$ & $66.64$ & $75.47$ & $68.29$ & $58.97$ \\
            & Caruana & $44.76$ & $\mathbf{84.08}$ & $39.79$ & $67.23$ & $83.15$ & $\mathbf{70.79}$ & $\mathbf{64.97}$ \\
            & Quality-Only & $\mathbf{44.94}$ & $80.52$ & $40.36$ & $\mathbf{68.35}$ & $82.87$ & $\mathbf{70.79}$ & $64.64$ \\
            & \textbf{\cmas{} (Ours)} & \scoregain{$44.19$}{5.48} & \scoregain{$80.52$}{12.14} & \scoregain{$\mathbf{42.04}$}{5.71} & \scoregain{$\mathbf{68.35}$}{1.71} & \scoregain{$\mathbf{83.90}$}{8.43} & \scoregain{$\mathbf{70.79}$}{2.50} & \scoregain{$\mathbf{64.97}$}{6.00} \\
            \midrule
            \multirow[c]{4}{*}{\rotatebox[origin=c]{90}{\textit{Stacking}}} & Random-$k$ & $41.04$ & $71.90$ & $38.43$ & $67.21$ & $78.81$ & $69.35$ & $61.12$ \\
            & Caruana & $46.35$ & $84.93$ & $43.26$ & $\mathbf{68.45}$ & $83.80$ & $\mathbf{70.69}$ & $66.25$ \\
            & Quality-Only & $46.54$ & $\mathbf{85.96}$ & $43.63$ & $68.35$ & $83.43$ & $\mathbf{70.69}$ & $66.43$ \\
            & \textbf{\cmas{} (Ours)} & \scoregain{$\mathbf{46.72}$}{5.68} & \scoregain{$\mathbf{85.96}$}{14.06} & \scoregain{$\mathbf{43.91}$}{5.48} & \scoregain{$68.35$}{1.14} & \scoregain{$\mathbf{84.83}$}{6.02} & \scoregain{$\mathbf{70.69}$}{1.34} & \scoregain{$\mathbf{66.74}$}{5.62} \\
            \bottomrule[1.25pt]
        \end{tabular*}
    \endgroup
\end{table}

\subsection{Team Composition Analysis}
\label{app:team_composition}

To further examine how heterogeneity-aware selection changes the selected team, we report in \Cref{tab:team_composition} the benchmarks where \cmas{} selects a different team from Quality-Only under \textsc{Stacking} aggregation. On the seven primary benchmarks, the two methods select identical teams on four benchmarks, with a mean Jaccard similarity of $0.79$. The three differing teams share the same pattern: \cmas{} keeps the strongest shared anchor (Nemotron, Gemma-2, or Qwen3) and substitutes one model from the central high-accuracy cluster with a peripheral one (GLM-4, Ministral-3). This pattern aligns with the heterogeneity landscape in \Cref{fig:mds}, where peripheral models provide complementary error coverage despite slightly lower individual accuracy.

\begin{table}[htbp]
    \centering
    \caption{Team composition changes made by \cmas{} relative to Quality-Only under \textsc{Stacking} aggregation. We report only benchmarks where the selected teams differ. $\Delta$ Acc. denotes the accuracy difference (pp) between \cmas{} and Quality-Only; arrows indicate gains or losses.}
    \label{tab:team_composition}
    \begingroup
    \small
    \setlength{\tabcolsep}{0.8pt}
    \renewcommand{\arraystretch}{1.18}
        \begin{tabular*}{\linewidth}{@{\extracolsep{\fill}}lccc@{}}
            \toprule[1.25pt]
            \textbf{Benchmark} & \textbf{Quality-Only Team} & \textbf{\cmas{} Team} & \textbf{$\Delta$ Acc.} \\
            \midrule[1.1pt]
            CommonsenseQA &
            \makecell[l]{\modeltag{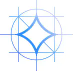}{Gemma-2} \\ \modeltag{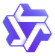}{Qwen3} \\ \removedtag{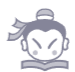}{InternLM3}} &
            \makecell[l]{\modeltag{figures/logos/gemma.png}{Gemma-2} \\ \modeltag{figures/logos/qwen.png}{Qwen3} \\ \addedtag{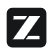}{GLM-4}} &
            \better{$\uparrow2.43$} \\
            MMLU-Pro &
            \makecell[l]{\modeltag{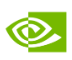}{Nemotron} \\ \modeltag{figures/logos/qwen.png}{Qwen3} \\ \removedtag{figures/logos/gemma.png}{Gemma-2}} &
            \makecell[l]{\modeltag{figures/logos/nemotron.png}{Nemotron} \\ \modeltag{figures/logos/qwen.png}{Qwen3} \\ \addedtag{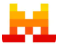}{Ministral-3}} &
            \better{$\uparrow1.78$} \\
            OpenBookQA &
            \makecell[l]{\modeltag{figures/logos/nemotron.png}{Nemotron} \\ \modeltag{figures/logos/gemma.png}{Gemma-2} \\ \removedtag{figures/logos/internlm.png}{InternLM3}} &
            \makecell[l]{\modeltag{figures/logos/nemotron.png}{Nemotron} \\ \modeltag{figures/logos/gemma.png}{Gemma-2} \\ \addedtag{figures/logos/glm.png}{GLM-4}} &
            \better{$\uparrow1.40$} \\
            \bottomrule[1.25pt]
        \end{tabular*}
    \endgroup
\end{table}

\subsection{Robustness and Ablation Analysis}
\label{app:robustness_ablation}

We examine the robustness of \cmas{} along two axes. \Cref{fig:additional_ablations} sweeps team size $k\in\{2,3,4,5\}$ under all four aggregation rules and shows that \cmas{} maintains a consistent margin over Quality-Only and Caruana across budgets, confirming that the gains reported with $k=3$ are not an artifact of a specific team size. \Cref{fig:objective_validity} then assesses whether the optimization objective is a faithful proxy for downstream accuracy: across the 13
benchmarks, the Spearman rank correlation between objective scores and ground-truth accuracy over all $\binom{11}{3}$ teams averages $\rho=0.751$, supporting the validity of our selection criterion.

\begin{figure}[htbp]
    \centering
    \begin{subfigure}{0.48\textwidth}
        \centering
        \includegraphics[width=\linewidth]{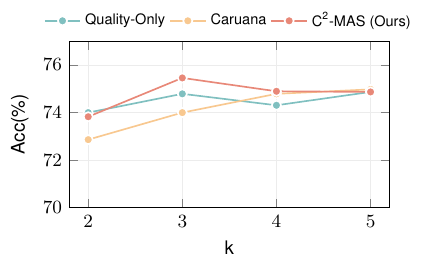}
        \caption{Team size sweep ($k$): Choice-Soft.}
        \label{fig:ablation_k_choice_soft}
    \end{subfigure}
    \hfill
    \begin{subfigure}{0.48\textwidth}
        \centering
        \includegraphics[width=\linewidth]{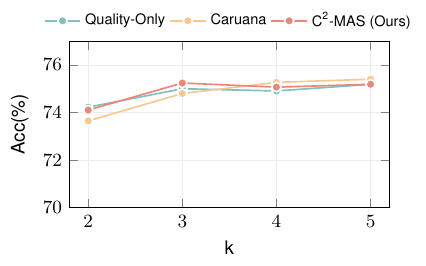}
        \caption{Team size sweep ($k$): PoE.}
        \label{fig:ablation_k_poe}
    \end{subfigure}

    \vspace{4pt}
    \begin{subfigure}{0.48\textwidth}
        \centering
        \includegraphics[width=\linewidth]{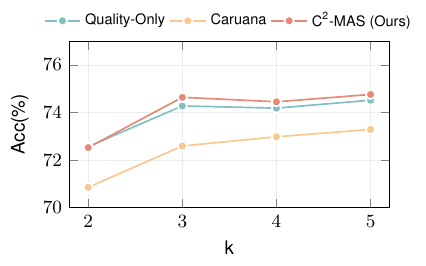}
        \caption{Team size sweep ($k$): DS.}
        \label{fig:ablation_k_ds}
    \end{subfigure}
    \hfill
    \begin{subfigure}{0.48\textwidth}
        \centering
        \includegraphics[width=\linewidth]{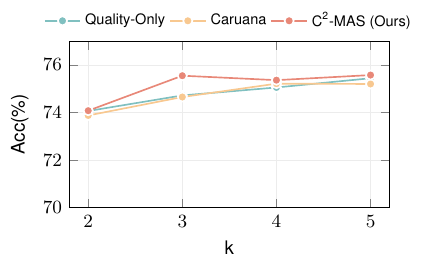}
        \caption{Team size sweep ($k$): Stacking.}
        \label{fig:ablation_k_stacking}
    \end{subfigure}
    \caption{Performance scaling with team size ($k$) across different aggregators.
    We compare the average test accuracy of \cmas{} (ours) against Quality-Only and Caruana baselines as the team size $k$ varies from 2 to 5. The evaluation is performed under four distinct inference-time aggregation strategies: (a) \textsc{Choice-Soft}, (b) \textsc{PoE}, (c) \textsc{DS}, and (d) \textsc{Stacking}.
    \cmas{} demonstrates consistent improvements over baselines across all strategies, particularly in \textsc{Stacking} and \textsc{DS}, indicating that our heterogeneity-aware selection is robust to the choice of downstream combination mechanism.
    }
    \label{fig:additional_ablations}
\end{figure}

\begin{figure}[htbp]
    \centering
    \includegraphics[width=0.96\linewidth]{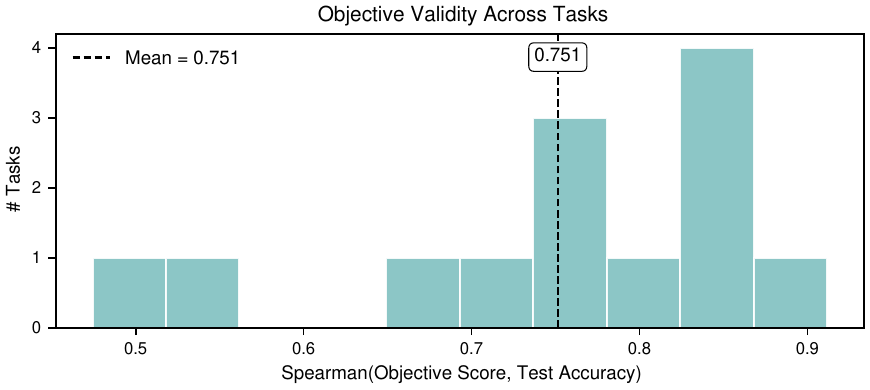}
    \caption{Validity of the optimization objective.
    To assess the reliability of our selection criterion, we compute the Spearman rank correlation coefficient ($\rho$) between the team scores derived from our objective function and their actual downstream test accuracy, calculated across all $\binom{11}{3}$ possible team combinations for each task.
    The histogram illustrates the distribution of these correlations across the 13 benchmarks.
    With a high mean correlation of \textbf{0.751}, the results confirm that our proposed objective serves as an effective proxy for ground-truth performance, allowing \cmas{} to accurately rank candidate teams based solely on profiling data.
    }
    \label{fig:objective_validity}
\end{figure}

\subsection{Stability of Pooled Heterogeneity Estimates}
\label{app:pooled_stability_results}

\Cref{fig:hi_matrix_stability} reports the variability of pairwise HI entries under repeated
subsampling of development examples. The reduced variability after pooling complements the
selection analysis: \Cref{app:pooling_geometry} maps changes in pooled, centered signals to
changes in standardized pairwise scores, and \Cref{prop:selection_stability} identifies when
those changes preserve the selected team.

\begin{figure}[htbp]
    \centering
    \includegraphics[width=\linewidth]{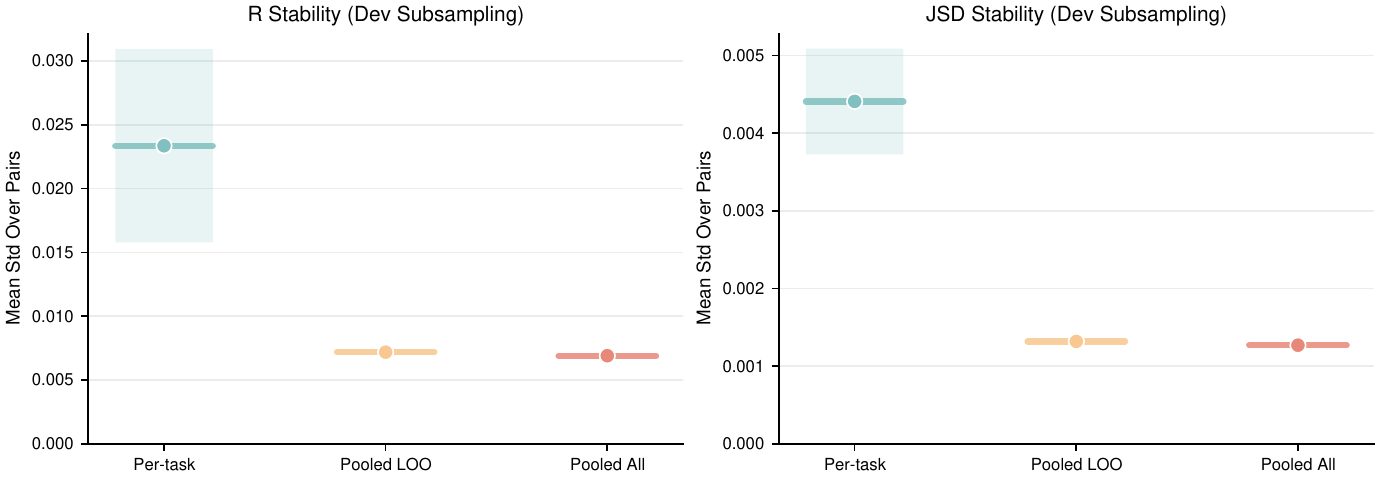}
    \caption{Empirical stability of pooled HI under dev subsampling. We subsample 70\% of
    each task's dev set for 50 repetitions, recompute per-task HI matrices (Yule's-$Q$ $R$ and JSD),
    and form pooled matrices from these noisy estimates. Pooling substantially reduces the
    across-repetition standard deviation of pairwise entries, complementing the selection-stability analysis in \Cref{app:selection_stability,app:pooling_geometry}.}
    \label{fig:hi_matrix_stability}
\end{figure}

\FloatBarrier
\section{Limitations}
\label{app:limitations}

Our current formulation of $\mathrm{HI}_{err}$ and $\mathrm{HI}_{dist}$ targets discriminative tasks with a shared finite label space, and does not directly cover open-ended generation, code synthesis, or multi-turn interactive settings. However, the decoupled pipeline extends naturally to generative settings by redesigning the metrics: $\mathrm{HI}_{err}$ can be adapted to pass@$k$ decorrelation (co-failures on unit tests), and $\mathrm{HI}_{dist}$ to embedding-based distances (e.g., BERTScore) between open-ended outputs. The selection procedure itself is unchanged, and we leave empirical study to future work.

Our profiling also assumes access to representative development data for each candidate model. Although cross-task pooling reduces estimation noise and improves transfer to held-out benchmarks, the selected team may still be sensitive to distribution shift when development tasks poorly match
deployment conditions.

Finally, \cmas{} addresses offline team selection rather than adaptive per-instance routing or interactive collaboration. The selected team is fixed for a task and can be paired with different aggregation rules, but our experiments do not study dynamic protocols such as debate, tool use, or iterative verification. Studying how the heterogeneity signals interact with such protocols is a natural next step.

\end{document}